\documentclass[a4paper]{article}

\usepackage[utf8x]{inputenc}
\usepackage[T1]{fontenc}

\usepackage[a4paper,margin=1in]{geometry}
\usepackage[numbers]{natbib}
\usepackage[english]{babel}
\usepackage[T1]{fontenc}
\usepackage{amsfonts, amsmath, amsthm}
\newcommand{\F}{\mathcal{F}}
\newcommand{\cH}{\mathcal{H}}
\newcommand{\thrsh}{\mathrm{THRESH}}
\usepackage[normalem]{ulem}
\newcommand{\strike}[1]{}
\usepackage{amsmath}
\usepackage{graphicx}
\usepackage[capitalise]{cleveref}
\usepackage{xspace}
\newcommand{\ignore}[1]{}
\newcommand{\fullversion}[1]{#1}
\newcommand{\W}{\mathcal{W}}

\newcommand{\EE}{\mathop\mathbb{E}}
\newcommand{\A}{\mathcal{A}}
\newcommand{\good}{\emph{simple\xspace}}
\renewcommand{\H}{\mathcal{H}}
\renewcommand{\P}{\mathbb{P}}

\newcommand{\ldim}{\mathrm{Ldim}}
\newcommand{\privacyexpression}{\bigl(6\tau \alpha(\tau |S|)+\tau,4e^{6\tau\alpha(\tau |S|)}\tau\beta(\tau |S|)\bigr)}
\newcommand{\mmd}{\mathrm{IPM}}
\newcommand{\gam}{\textrm{Sequential}}
\newcommand{\creative}{\textrm{DP}\xspace}
\newcommand{\syn}{p_{syn}}
\newcommand{\preal}{p_{real}}

\newcommand{\regret}{\mathrm{REGRET}_T}
\newcommand{\X}{\mathcal{X}}
\newcommand{\D}{\mathcal{D}}
\newcommand{\T}{\mathcal{T}}
\newcommand{\eps}{\epsilon}

\newcommand{\win}{\mathrm{WIN}}

\crefname{subsubsubappendix}{Appendix}{Appendices}
\newcommand{\shay}[1]{}
\newcommand{\roi}[1]{}
\newcommand{\olivier}[1]{}
\newtheorem{theorem}{Theorem}
\newtheorem{definition}{Definition}
\newtheorem{lemma}{Lemma}
\newtheorem{corollary}{Corollary}
\newtheorem{question}{Question}
\newtheorem{claim}{Claim}
\newtheorem{observation}{Observation}
\newtheorem{proposition}{Proposition}

\newcommand{\ignorenips}[1]{}
\renewcommand{\fullversion}[1]{#1}
\title{Synthetic Data Generators -- Sequential and Private}

\author{Olivier Bousquet\thanks{Google Brain, Z\"urich. {\tt obousquet@google.com.}} \and Roi Livni\thanks{Department of Electrical Engineering, Tel Aviv University, Tel Aviv. {\tt rlivni@tauex.tau.ac.il}}\and Shay Moran\thanks{Department of Mathematics, Technion, Haifa. {\tt shaymoran1@gmail.com.}}}

\begin{document}

\maketitle

\begin{abstract}
We study the sample complexity of private synthetic data generation over an unbounded sized class of statistical queries, and show that any class that is privately proper PAC learnable admits a private synthetic data generator (perhaps non-efficient). Previous work on synthetic data generators focused on the case that the query class $\D$ is finite and obtained sample complexity bounds that scale logarithmically with the size $|\D|$. Here we construct a private synthetic data generator whose sample complexity is independent of the domain size, and we replace finiteness with the assumption that $\D$ is privately PAC learnable (a formally weaker task, hence we obtain equivalence between the two tasks).

\end{abstract}

\section{Introduction}
Generating differentially--private synthetic data \cite{blum2013learning, dwork2009complexity} is a fundamental task in learning that has won considerable attention in the last few years \cite{hardt2012simple, ullman2011pcps, hsu2013differential,gaboardi2014dual}.
Formally, given a class $\D$ of distinguishing functions, a fooling algorithm receives as input IID samples from an unknown real-life distribution, $p_{real}$, and outputs a distribution $p_{syn}$ that is $\epsilon$-close to $p_{real}$ w.r.t the \emph{Integral Probability Metric} (\cite{muller1997integral}), denoted $\mmd_{\D}$:

\begin{equation}\label{eq:mmd}
\mmd_{\D}(p,q) = \sup_{d\in \D} \left|\EE_{x\sim p}[d(x)]- \EE_{x\sim q}[d(x)]\right|
\end{equation}

A DP-SDG is then simply defined to be a differentially private fooling algorithm. 

A fundamental question is then: Which classes $\D$ can be privately fooled? In this paper, we focus on sample complexity bounds and give a first such characterization. We prove that a class $\D$ is DP--foolable if and only if it is privately (proper) PAC learnable. As a corollary, we obtain equivalence between several important tasks within private learning such as proper PAC Learning \cite{kasiviswanathan2011can}, Data Release \cite{dwork2009complexity}, Sanitization \cite{beimel2013private} and what we will term here \emph{Private Uniform Convergence}.

Much focus has been given to the task of synthetic data generation. Also, several papers \citep{bassily2020private, hsu2013differential,gaboardi2014dual, gupta2013privately, gupta2012iterative} discuss the reduction of private fooling to private PAC learning. In contrast with previous work, we assume an arbitrary large domain. In detail, previous existing bounds normally scale logarithmically with the size of the query class $\D$ (or alternatively, depend on the size of the domain). Here we initiate a study of the sample complexity that does not assume that the size of the domain is fixed. Instead, we only assume that the class is privately PAC learnable, and obtain sample complexity bounds that are independent of the cardinality $|\D|$. We note that the existence of a private synthetic data generator entails private proper PAC learning, hence our assumption is a necessary condition for the existence of a DP-SDG.

The general approach taken for generating synthetic data (which we also follow here) is to exploit an online setup of a sequential game between a generator that aims to fool a discriminator and a discriminator that attempts to distinguish between real and fake data. The utility and generality of this technical method, in the context of privacy, has been observed in several previous works \cite{hardt2012simple,niel2018how, gupta2013privately}.
%Then, one often relies on an online algorithm that converges to an optimal strategy in terms of the generator. Luckily, online algorithms for privately learnable classes are known to exist (In fact, a necessary condition for a class to be privately learnable is that it is online learnable \cite{alon}). Therefore we also take the above approach, and as we elaborate later, we introduce a general setup of sequential SDGs that are utilized exactly in this manner. 
However, in the finite case, specific on-line algorithms, such as \emph{Multiplicative Weights} \cite{gupta2013privately} and \emph{Follow-the-Perturbed-Leader} \cite{vietri2020new} are considered. The algorithms are then exploited, in a white-box fashion, that allow easy construction of SDGs. The technical challenge we face in this work is to generalize the above technique in order to allow the use of no-regret algorithms that work over infinite classes. Such algorithms don't necessarily share the attractive traits of MW and FtPL that allow their exploitation for generating synthetic data. To overcome this, we study here a general framework of \emph{sequential SDGs} and show how an \emph{arbitrary} online algorithm can be turned, via a Black-box process, into an SDG which in turn can be privatized. We discuss these challenges in more detail in \fullversion{\cref{sec:uboverview}}.

Thus, the technical workhorse behind our proof is a learning primitive which is of interest of its own right. We term it here \emph{Sequential Synthetic Data Generator} (Sequential-SDG). Similar frameworks appeared \cite{gupta2013privately, vietri2020new} in the context of private-SDGs but also more broadly in the context of generative learning \cite{levy, abernethy, goodfellow2014generative, goodfellow2016nips}. We further discuss this deep and important connection between private learning and generative learning in \cref{sec:discussion}

In the sequential-SDG setting, we consider a sequential game between a generator (player G) and a discriminator (player D). 
At every iteration, player G proposes a distribution and player D outputs a discriminating function from a prespecified binary class $\D$. 
The game stops when player G proposes a distribution that is close in $\mmd_{\D}$ distance to the true target distribution. As we focus on the statistical limits of the model, we ignore the optimization and computational complexity aspects and we assume that both players are omnipotent in terms of their computational power.

We provide here characterization of the classes that can be \emph{sequentially fooled} (i.e. classes $\D$ for which we can construct a sequential SDG) and show that the sequentially foolable classes are exactly \emph{Littlestone classes} \cite{littlestone, bendavid}. In turn, we harness sequential SDGs to generate synthetic data together with a private discriminator in order to generate private synthetic data. Because this framework assumes only a private learner, we in some sense show that the sequential setting is a canonical method to generate synthetic data.

To summarize this work contains several contributions: We provide the first domain-size independent sample complexity bounds for DP-Fooling, and show an equivalence between private synthetic data generation and private learning. Second, we introduce and characterize a new class of SDGs and demonstrate their utility in the construction of private synthetic data.

\section{Prelimineries}\label{sec:preliminaries}
In this section we recall standard definitions and notions in differential privacy and learning (a more extensive background is also given in \fullversion{\cref{sec:background}}). 
Throughout the paper we will study classes $\D$ of boolean functions defined on a domain $\X$.
However, we will often use a dual point of view where we think of $\X$ as the class of functions and on $\D$ as the domain. Therefore, in order to avoid confusion, in this section we let $\W$ denote the domain and $\H\subseteq\{0,1\}^W$ to denote the functions class.
%
%and a class of discriminating functions. We will heavily exploit duality, and we will sometime think of $\X$ as a hypothesis class over $\D$ and vice versa. Thus, we will use in this section $\W$ to denote a domain and $\H\subseteq \{0,1\}^{\W}$ to denote a hypothesis class.}
%to
\ignore{ 
\subsection{Notations}\label{sec:basics}

For a finite\footnote{The same notation will be used for infinite classes also. However we will properly define the the measure space and $\sigma$-algebra at later sections when we extend the results to the infinite regime.} set $\W$, let $\Delta(\W)$ denote the space of probability measures over $\W$. 
	Note that $\W$ naturally embeds in $\Delta(\W)$ by identifying $w\in \W$
	with the Dirac measure $\delta_w$ supported on $w$.
	Therefore, every $f:\Delta(\W)\to\mathbb{R}$ induces a $\W\to\mathbb{R}$ function via this identification.
	In the other direction, every~$f:\W\to\mathbb{R}$ naturally extends to a linear\footnote{A function $g:\Delta(\W)\to\mathbb{R}$ is {\it linear}
	if $g\bigl(\alpha p_1 + (1-\alpha) p_2\bigr) = \alpha g(p_1) + (1-\alpha)g(p_2)$, for all $\alpha\in[0,1]$} 
	map $\hat f:\Delta(\W)\to \mathbb{R}$ which is defined by~$\hat f(p) = \EE_p[f]$	for every~$p\in \Delta(\W)$.

We will often deal with boolean functions $f:\W\to\{0,1\}$,
	and in some cases we will treat $f$ as the subset of $\W$ that it indicates.
	For example, given a distribution $p\in\Delta(\W)$  we will use $p(f)$ to denote the measure of the subset that $f$ indicates 
	(i.e.\ $p(f) = \Pr_{w\sim p }[f(w) = 1]$).
	Given a class of functions $F\subseteq \{0,1\}^\W$, 
	its {\it dual class} is a class of $F\to\{0,1\}$ functions,
	where each function in it is associated with $w\in \W$
	and acts on $F$ according to the rule $f\mapsto f(w)$.
	By a slight abuse of notation we will denote the dual class with $\W$ 
	and use $w(f)$ to denoted the function associated with $w$ (i.e.\ $w(f):=f(w)$ for every $f\in F$).
	%\shay{Should we use $\W$ throughout the preliminaries and only introduce $\X$ when we discuss the model? We should also note that we  only focus on statistical efficiency and ignore other computational aspects.}
	
	Given a sample $S=(w_1,\ldots,w_m)\in \W^m$,
	the {\it empirical distribution} induced by $S$ is the discrete distribution $p_S$ defined by
	$p_S(w)=\frac{1}{m}\sum_{i=1}^{m}1[w=w_i].$

}
\subsection{Differential Privacy and Private Learning}\label{sec:privacy}
Differential Privacy~\cite{dwork1,dinur} is a statistical formalism which aims at capturing algorithmic privacy.
It concerns with problems whose input contains databases with private records and it enables to design algorithms that are formally guaranteed to protect the private information. 	For more background see the surveys \cite{dwork3,vadhan}.
The formal definition is as follows: 
	let $\W^m$ denote the input space. An input instance~$\Omega \in \W^m$ is called a {\it database},
	and two databases $\Omega',\Omega''\in \W^m$ are called neighbours if there exists a single $i\leq m$
	such that $\Omega'_i\neq \Omega''_i$.
	Let $\alpha,\beta>0$ be the privacy parameters,
	a randomized algorithm $M:\W^m\to \Sigma$ is called $(\alpha,\beta)$-differentially private
	if for every two neighbouring $\Omega',\Omega''\in \W^m$ and for every event~$E\subseteq \Sigma$:
	\[
	\Pr\bigl[M(\Omega')\in E\bigr] \leq e^\alpha \Pr\bigl[M(\Omega'')\in E\bigr] + \beta\ignore{~~\text{ and }~~\Pr\bigl[M(\Omega'')\in E\bigr] \leq e^\alpha \Pr\bigl[\Omega(X')\in E\bigr] + \beta}.
	\]
	An algorithm $M:\cup_{m=1}^{\infty}\W^m\to Y$ is called differentially private if for every $m$ 
	its restriction to $\W^m$ is $(\alpha(m),\beta(m))$-differentially private,
	where $\alpha(m) = O(1)$ and $\beta(m)$ is negligible\footnote{I.e.\ $\beta(m) = o(m^{-k})$ for every $k>0$.}.
	Concretely,  we will think of $\alpha(m)$ as a small constant (say, $0.1$) and $\beta(m) = O(m^{-\log m})$.

\paragraph{Private Learning.}
We next overview the notion of Differentially private learning algorithms \cite{kasiviswanathan2011can}. In this context the input database is the training set of the algorithm.
Given a hypothesis class $\H$ over a domain $W$, we say that $\H\subseteq \{0,1\}^\W$ is privately PAC learnable if it can be learned by a differentially private algorithm.
	That is, if there is a differentially private algorithm~$M$ and a sample complexity bound $m(\eps,\delta)=\mathrm{poly}(1/\eps,1/\delta)$ such that for every $\eps,\delta>0$
	and every distribution $\P$ over $\W\times\{0,1\}$, if $M$ receives an independent sample $S\sim\P^m$ 
	then it outputs an hypothesis $h_S$ such that with probability at least $1-\delta$:
	\[ L_\P(h_S)\le \min_{h\in \H} L_\P(h) + \epsilon,\]
	where $L_\P(h) = \EE_{(w,y)\sim \P} \bigl[1[h(w)\ne y]\bigr]$. If $M$ is {\it proper}, namely $h_S\in \H$ for every input sample $S$,
	then $\H$ is said to be Privately Agnostically and Properly PAC learnable (PAP-PAC-learnable).

In some of our proofs it will be convenient to consider private learning algorithms whose privacy parameter $\alpha$
	satisfies $\alpha\leq 1$ (rather than $\alpha = O(1)$ as in the definition of private algorithms).
	This can be done without loss of generality due to privacy amplification theorems (see, for example for example \cite{vadhan} (discussion after definition 8.2 therein, and see also \fullversion{discussion after \cref{lem:amplification}} for further details).

\paragraph{Sanitization.}
The notion of sanitization has been introduced by \citet{blum2013learning} and further studied in \cite{beimel2013private}. 
	Let $\H\subseteq\{0,1\}^\W$ be a class of functions.
	An {\it $(\epsilon,\delta,\alpha,\beta,m)$-sanitizer} for $\H$ is an $(\alpha,\beta)$-private algorithm $M$  
	that receives as an input a sample $S \in \W^m$ and outputs a function $\mathrm{Est}: \H\to [0,1]$
	such that with probability at least $1-\delta$,
	\[(\forall h\in\H): \Bigl\lvert\mathrm{Est}(h)- \frac{\lvert\{w\in S: h(w)=1\}\rvert}{|S|}\Bigr\rvert\le \epsilon.\]
\ignore{	A common type of sanitizers acts as follows:
	given an input sample $S$, they publish a sample $T$
	such that with probability at least $1-\delta$,
	\[(\forall h\in\H): \Bigl\lvert \frac{\lvert\{w\in T: h(w)=1\}\rvert}{|T|} - \frac{\lvert\{w\in S: h(w)=1\}\rvert}{|S|} \Bigr\rvert \leq \epsilon.\]
	In other words, they use the estimate $\mathrm{Est}(h) = \frac{\lvert\{w\in T: h(w)=1\}\rvert}{|T|}$, which is encoded by $T$.}

We say that $\H$ is \emph{sanitizable} if there exists an algorithm $M$ and a bound $m(\eps,\delta)=\mathrm{poly}(1/\eps,1/\delta)$
	such that for every $\eps,\delta>0$, the restriction of $M$ to samples of any size $m\ge m(\eps,\delta)$
	is an $(\epsilon,\delta,\alpha,\beta,m)$-sanitizer for $\H$ with $\alpha=\alpha(m)=O(1)$ and $\beta=\beta(m)$ negligible.

\paragraph{Private Uniform Convergence.}
A basic concept in Statistical Learning Theory is the notion of \emph{uniform convergence}. 
	In a nutshell, a class of hypotheses $\H$ satisfies the uniform convergence property 
	if for any unknown distribution $\P$ over examples, 
	one can uniformly estimate the expected losses of all hypotheses in $\H$ given a large enough sample from $\P$. 
	Uniform convergence and statistical learning are closely related. For example, 
	the {\it Fundamental Theorem of PAC Learning} asserts that 
	they are equivalent for binary-classification~\cite{Shalev14book}.

This notion extends to the setting of private learning:
	a class $\H$ satisfies the {\it Private Uniform Convergence} property if there exists a differentially private algorithm $M$ and a sample complexity bound 
	$m(\epsilon,\delta)=\mathrm{poly}(1/\eps,1/\delta)$ such that for every distribution~$\P$ over $\W\times\{0,1\}$ 
	the following holds: if  $M$ is given an input sample $S$ of size at least $m(\eps,\delta)$ which is drawn independently from~$\P$, 
	then it outputs an estimator $\hat L:\H\to[0,1]$ such that with probability at least $(1-\delta)$ it holds that
	\[(\forall h\in \H) : \bigl\lvert \hat L(h) - L_\P(h) \big\rvert \leq \eps.\]
	Note that without the privacy restriction, the estimator 
	\[\hat L(h) = L_S(h) := \frac{\lvert\{(w_i,y_i)\in S: h(w_i)\ne y_i\}\rvert}{|S|}\]
	satisfies the requirement for $m=\tilde O({d}/{\eps^2})$, where $d$ is the VC-dimension of $\H$;
	this follows by the celebrated VC-Theorem \cite{Vapnik71uniform,Shalev14book}.

\section{Problem Setup}\label{sec:model}
We assume a domain $\X$ and we let~$\D \subseteq \{0,1\}^\X$ be a class of functions over $\X$. 
The class $\D$ is referred to as the \emph{discriminating functions class} and its members $d\in \D$ are called \emph{discriminating functions} or {\it distinguishers}. We let $\Delta(\X)$ denote the space of distributions over $\X$. Given two distributions $p,q\in \Delta(\X)$, let $\mmd_{\D}(p,q)$ denote the $\mmd$ distance between $p$ and $q$ as in \cref{eq:mmd}.

It will be convenient to assume that $\D$ is \emph{symmetric}, i.e.\ that whenever $d\in \D$ then also its complement, $1-d\in \D$. Assuming that $\D$ is symmetric will not lose generality and will help simplify notations. We will also use the following shorthand: given a distribution $p$ and a distinguisher $d$ we will often write
\begin{align*}
p(d):= \EE_{x\sim p} [d(x)].
\end{align*}
 Under this assumption and notation we can remove the absolute value from the definition of $\mmd$:
\begin{align}
\mmd_{\D}(p,q)
&= \sup_{d\in \D}\left(p(d)- q(d)\right).
\end{align}

\subsection{Synthetic Data Generators}
A synthetic data generator (SDG), without additional constraints, is defined as follows

\begin{definition}[SDG]
An SDG, or a fooling algorithm, for $\D$ with sample complexity $m(\epsilon,\delta)$  
is an algorithm $M$ that receives as input a sample $S$ of points from $\X$ and parameters $\epsilon,\delta$ such that the following holds:
for every $\epsilon,\delta>0$ and every target distribution $\preal$, 
if $S$ is an independent sample of size at least $m(\eps,\delta)$ from~$\preal$ 
then 
\[
\Pr\Bigl[ \mmd_{\D}(\syn,\preal)<\epsilon \Bigr]\geq 1-\delta,
\]
where $\syn := M(S)$ is the distribution outputted by $M$, and the probability
is taken over $S\sim (\preal)^m$ as well as over the randomness of $M$.
%
%with probability at least $1-\delta$ (over the randomness of $M$ and the sample $S$) 
%the algorithm outputs a distribution $\syn$ such that:
%\begin{align}
%\mmd_{\D}(\syn,\preal)<\epsilon.\end{align}
\end{definition}
We will say that a class is \emph{foolable} if it can be fooled by an SDG algorithm whose sample complexity is $\mathrm{poly}(\frac{1}{\epsilon},\frac{1}{\delta})$. Foolability, without further constraints, comes with the following characterization which is an immediate corollary (or rather a reformulation) of the celebrated VC Theorem (\cite{Vapnik71uniform}). 

Denote by $M_{emp}$ an algorithm that receives a sample $S$ and returns $M_{emp}(S):=p_S$, the empirical distribution over $S$.
\begin{observation}[\cite{Vapnik71uniform}]\label{obs:vc}
The following statements are equivalent for a class $\D\subseteq \{0,1\}^\X$:
\begin{enumerate}
\item $\D$ is PAC--learnable.
\item $\D$ is foolable.
\item $\D$ satisfies the uniform convergence property.
\item $\D$ has a finite VC-dimension.
\item $M_{emp}$ is a fooling algorithm for $\D$ with sample complexity $m=O(\frac{\log1/\delta}{\eps^2})$.
\end{enumerate}
\end{observation}
\cref{obs:vc} shows that foolability is equivalent to PAC-learnability (and in turn to \ finite VC dimension).
We will later see analogous results for $\creative$--Foolability (which is equivalent to differentially private PAC learnability)
and $\gam$--Foolability (which is equivalent to online learnability). 

We now discuss the two fundamental models that are the focus of this work -- $\creative$--Foolability and $\gam$--Foolability.

\subsection{$\creative$--Synthetic Data Generators}
We next introduce the notion of a $\creative$--synthetic data generator and $\creative$--Foolability. As discussed, DP-SDGs have been the focus of study of several papers \citep{blum2013learning, dwork2009complexity, hardt2012simple, ullman2011pcps, hsu2013differential,gaboardi2014dual}.

\begin{definition}[DP-SDG]
 A DP-SDG, or a DP-fooling algorithm $M$ for a class $\D$  is an algorithm that receives as an input a finite sample $S$ and two parameters $(\epsilon,\delta)$ and satisfies:
 \begin{itemize}
     \item {\bf Differential Privacy.}
     For every $m$, the restriction of $M$ to input samples $S$ of size $m$ 
     is $(\alpha(m),\beta(m))$-differentially private, where $\alpha(m) = O(1)$ and $\beta(m)$ is negligible.
      \item {\bf Fooling.} $M$ fools $\D$: there exists a sample complexity bound $m=m(\eps,\delta)$
      such that for every target distribution $\preal$ if $S$ is a sample of at least $m$ examples from $\preal$
      then $\mmd_{\D}({\syn,\preal})\leq \eps$ with probability at least $1-\delta$, where $\syn$ is the output of $M$ on the input sample $S$.
 \end{itemize}
 \end{definition}
	
We will say in short that a class $\D$ is \creative -- Foolable if there exists a DP-SDG for the class $\D$ with sample complexity $m=\mathrm{poly}(1/\eps,1/\delta)$.

\subsection{$\gam$--Synthetic Data Generators}
We now describe the second model of foolability which, as discussed, is the technical engine behind our proof of equivalence between DP-foolability and DP-learning.

\paragraph{Sequential-SDGs}
A Sequential-SDG can be thought of as a sequential game between two players 
	called the {\it generator} (denoted by $G$) and the {\it discriminator} (denoted by $D$).
	At the beginning of the game, the discriminator $D$ receives the target distribution which is denoted by $\preal$. 
	The goal of the generator $G$ is to find a distribution $p$ such that $p$ and $\preal$ are $\eps$-indistinguishable
	with respect to some prespecified discriminating class $\D$ and an error parameter~$\eps>0$, i.e.\ 
	\[\mmd_\D(p,\preal)\leq\eps.\]
	We note that both players know $\D$ and $\eps$. The game proceeds in rounds, where in each round $t$ the generator $G$ submits to the discriminator a candidate distribution $p_t$
	and the discriminator replies according to the following rule: if $\mmd_\D(p_t,\preal)\leq\eps$ then the discriminator replies ``$\win$''
	and the game terminates. Else, the discriminator picks $d_t\in \D$ such that $|\preal(d_t)-p_t(d_t)|>\eps$,
	and sends $d_t$ to the generator along with a bit which indicates whether $p_t(d_t)> \preal(d_t)$ or $p_t(d_t) < \preal(d_t)$.
	Equivalently, instead of transmitting an extra bit, we assume that the discriminator always sends $d_t\in \D\cup (1-\D)$ s.t.
	\begin{align}\label{eq:win}
	\preal(d_t) -   p_t(d_t) >   \epsilon.
	\end{align}
%\new{As noted in the beginning of this section, we will assume that $\D$ is symmetric (i.e.\ that $1-\D = \D$).
%This does not affect generality; in particular, one can symmetrize $\D$ without\footnote{Indeed, one can verify that the dual Littlestone dimension of $\D\cup (1-\D)$
%is the same like the dual Littlestone dimension of $\D$.} increasing its dual Littlestone dimension
%which, as we will see, is the parameter that governs the number of rounds in the $\gam$ game.
%}

\begin{definition}[$\gam$--Foolability]
Let $\eps>0$ and let $\D$ be a discriminating class. 
\begin{enumerate}
\item $\D$ is called $\epsilon$-$\gam$--Foolable if there exists a generator $G$ 
	and a bound~$T=T(\eps)$ such that $G$ wins any discriminator $D$ with any target distribution~$\preal$
	after at most $T$ rounds.
\item The \emph{round complexity} of $\gam$--Fooling $D$ is defined as the minimal upper bound $T(\eps)$ on the number of rounds 
	that suffice to $\eps$--Fool $\D$.
\item $\D$ is called $\gam$--Foolable if it is $\eps$-$\gam$ foolable for every $\eps>0$ with $T(\eps)=\mathrm{poly}(1/\eps)$. 
\end{enumerate}
\end{definition}

In the next section we will see that if $\D$ is $\eps$-$\gam$--Foolabe for some fixed~$\eps < 1/2$
	then it is $\gam$--Foolable with round complexity $T(\eps) = O(1/\eps^2)$.
\ignorenips{
\paragraph{Randomness.} 
We stress out that we assume that both the generator and discriminator are deterministic. 
	This assumption is made to simplify the presentation but does not restrict the validity of our results, as we explain next.
	Consider a setting where the players may use randomness
	and the definition of round complexity is modified by taking expectation.
	Assuming a deterministic discriminator does not lose generality because in each round 
	the discriminator plays only after seeing the generator's submitted distribution and so it may always respond deterministically.
	Restricting the generator to be deterministic is more subtle.
	However, in terms of upper bounds it only strengthen our result (because we only use deterministic generators, whereas randomized ones could potentially be stronger),
	and in terms of our lower bound (\cref{thm:main1quant}, \cref{it:lb}), its proof applies verbatim to the expected round complexity of randomized generators.
}

\section{Results}

%\subsection{$\creative$--Foolability}

Our main result characterizes \creative--Foolability
in terms of basic notions from differential privacy and PAC learning.

\begin{theorem}[Characterization of \creative--Fooling]\label{thm:main3}
The following statements are equivalent for a class $\D\subseteq\{0,1\}^X$:
\begin{enumerate}
\item\label{it:pac} $\D$ is privately and properly learnable in the agnostic PAC setting.
\item\label{it:foolable} $\D$ is \creative--Foolable.
\item\label{it:sanitizable} $\D$ is sanitizable.
\item\label{it:puc} $D$ satisfies the private uniform convergence property.
\end{enumerate}
\end{theorem}
\cref{thm:main3} shows a qualitative equivalence between the relevant four notions, quantitative bounds on the entailed sample complexity are provided in \fullversion{\cref{lem:quant:1,lem:quant:2,lem:quant:3,lem:quant:4}}.

The implication \cref{it:sanitizable} $\implies$ \cref{it:pac} was known prior to this work and was proven in \cite{beimel2013private} (albeit the pure case\footnote{We note though, that while the proof provided in \cite{beimel2013private} is restricted to the pure case, our proof here doesn't sidetrack from their original proof}). The equivalence among \cref{it:foolable,it:sanitizable,it:puc} is natural and expected.
	Indeed, each of them expresses the existence of a private algorithm that  {\it publishes, privately, certain estimates of all functions in $\D$}.
%	The implication \cref{it:puc} $\implies$ \cref{it:pac} is straight-forward and follows by first using \cref{it:puc}
%	to privately publish estimates of all losses and then output an hypothesis in $\D$ that minimizes the estimated loss. 
%	The privacy of the resulting algorithm follows as privacy is preserved under {\it post-processing}. 
	
The fact that \cref{it:pac} implies the other three items is perhaps more surprising, and the main contribution of this work, and we show that \cref{it:pac} implies \cref{it:foolable}.
	Our proof of that exploits the $\gam$ framework. In a nutshell, we observe that a class that is both sequentially foolable and privately pac learnable is also DP-foolable: this result follows by constructing a sequential SDG that with a private discriminator, that is assumed to exists, combined with standard compositional and preprocessing arguments regarding the privacy of the generators output.
	
	Thus to prove the implication we only need to show that private PAC learning implies sequential foolability. This result follows from \cref{thm:main1} that provides characterization of sequential foolable classes as well as a recent result by \citet{alon} that shows that private PAC learnable classes have finite Littlestone dimension. See \fullversion{\cref{sec:main3}} for a complete proof.

\paragraph{Private learnability versus private uniform convergence.}	
The equivalence \cref{it:pac}$\iff$\cref{it:puc} is between private learning and private uniform convergence.
	The non-private analogue of this equivalence is a cornerstone in statistical learning;
 	it reduces the statistical challenge of minimizing an unknown population loss 
	to an optimization problem of minimizing a known empirical estimate.
	In particular, it yields the celebrated {\it Empirical Risk Minimization} (ERM) principle: 
	{\em ``Output~$h\in \H$ that minimizes the empirical loss''}.
	We therefore highlight this equivalence in the following corollary:
\begin{corollary}[Private proper learning =  private uniform convergence]
Let~$\H\subseteq\{0,1\}^\X$. Then $\H$ is privately and properly PAC learnable if and only if
	$\H$ satisfies the private uniform convergence property.
\end{corollary}
%\strike{It may be interesting to note that although the above result lies in {\it statistical learning theory},
%	our proof of it heavily exploits machinery from {\it online learning theory} 
%	(e.g.\ agnostic online learnability of Littlestone classes). }

\paragraph{$\gam$--SDGs}

We next describe our characterization of $\gam$-SDGs. As discussed, this characterization is the technical heart behind the equivalence between private PAC learning and DP-foolability. Nevertheless we believe that it may be of interest of its own right. We thus provide quantitative upper and lower bounds on the round complexity of $\gam$-SDGs in terms of the Littlestone dimension (see \cite{bendavid} or \fullversion{\cref{sec:background}} for the exact definition).

\begin{theorem}[Quantitative round-complexity bounds]\label{thm:main1quant}
Let $\D$ be a discriminating class with dual Littlestone dimension $\ell^*$
and let $T(\eps)$ denote the round complexity of $\gam$--Fooling $\D$.
Then,
\begin{enumerate}
\item\label{it:ub} $T(\eps)= O\bigl(\frac{\ell^*}{\eps^2}\log \frac{\ell^*}{\eps} \bigr)$ for every $\eps$.
\item\label{it:lb} $T(\eps)\geq\frac{\ell^*}{2}$ for every $\eps<\frac{1}{2}$. 
\end{enumerate}
%Moreover, the upper bound in \cref{it:ub} is exhibited by a deterministic generator.
\end{theorem}
It would be interesting to close the gap between the two bounds in terms of $\epsilon>0$, and we leave it for future work.
To prove \Cref{it:ub} we construct a generator with winning strategy which we outline in \fullversion{\Cref{sec:uboverview}}. 
	A complete proof of \cref{thm:main1quant} appears in \fullversion{\cref{prf:main_up}}. As a corollary we get the following characterization of $\gam$--Foolability:
\begin{corollary}[Characterization of \gam--Foolability]\label{thm:main1}
The following are equivalent for $\D\subseteq\{0,1\}^X$: 
\begin{enumerate}
\item\label{it:learnable} $\D$ is \gam--Foolable.
%\item $\D$ is $\eps$-$\gam$--Foolable for every $\epsilon > 0$,
\item\label{it:exist} $\D$ is $\eps$-$\gam$--Foolable for some $\eps <1/2$.
\item\label{it:finite_dual} $\D$ has a finite dual Littlestone dimension.
\item\label{it:finite_primal} $\D$ has a finite Littlestone dimension.
%\item\label{it:boosting} $\D$ is $\gam$ learnable with $T(\epsilon)$ steps for some $T$ such that for $\epsilon_0=1/2$ we have $T(\epsilon_0)<\infty$.
\end{enumerate}
\end{corollary}

Corollary~\ref{thm:main1} follows directly from \cref{thm:main1quant} (which gives the equivalences  $1\iff 2 \iff 3$) and from \cite{bhaskar2017thicket}  (which gives the equivalence $3 \iff 4$, see \fullversion{\cref{thm:little_dual}} for further detail).

%The equivalence between \cref{it:finite_dual} and \cref{it:finite_primal} in \cref{thm:main1} was known prior to this work~\cite{}.
\paragraph{Tightness of $\eps=\frac{1}{2}$.}
The implication \cref{it:exist} $\implies$ \cref{it:learnable} can be seen as a boosting result:
	i.e.\ ``weak'' foolability for some fixed $\eps<1/2$ implies ``strong'' foolability for every $\eps$.
	The following example demonstrates that the dependence on $\eps$ in \cref{it:exist} can not be improved beyond $\frac{1}{2}$:
	let~$\X$ be the unit circle in $\mathbb{R}^2$, and let $\D$ consist of all arcs whose length is exactly half of the circumference.
	It is easy to verify that the uniform distribution $\mu$ over $\X$ satisfies $\mmd_\D(\mu,\preal)\leq \frac{1}{2}$ for any target distribution $\preal$
	(since $\mu(d)=\frac{1}{2}$ for all $d\in\D$).
	Therefore $\D$ is $(\eps=\frac{1}{2})$-$\gam$--Foolable with round complexity $T(\frac{1}{2})=1$.
	On the other hand, $\D$ has an infinite Littlestone dimension and therefore is not $\gam$--Foolable.

\paragraph{$\gam$-SDGs versus $\creative$-SDGs}
So far we have introduced and characterized two formal setups for synthetic data generation. It is therefore natural to compare and seek connections between these two frameworks.
	We first note that the \creative setting may only be more restrictive than the $\gam$ setting:
\begin{corollary}[$\creative$--Foolability implies $\gam$--Foolability] \label{thm:main2}
Let $\D$ be a class that is $\creative$--Foolable. Then $\D$ has finite Littlestone dimension and in particular is $\gam$--Foolable.
\end{corollary}

\cref{thm:main2} follows from \cref{thm:main3}: indeed, the latter yields that \creative--Foolability is equivalent
	to Private agnostic proper -PAC learnability (PAP-PAC), and by \cite{alon} PAP-PAC learnability implies a finite Littlestone dimension which by \cref{thm:main1} implies $\gam$--Foolability.
\ignorenips{
\paragraph{The universality of $\gam$s-based \creative algorithms.}
Our proofs reveal a stronger phenomenon than \Cref{thm:main2}.
	In more detail, \Cref{thm:main3} shows that DP-Foolability and PAP-PAC are equivalent. 
	Proving the direction that PAP-PAC implies DP-Foolability obtains a \creative algorithm by invoking a  $\gam$-SDG with a differentially private discriminator. 
	This establishes universality of such $\gam$-based algorithms in the following sense: 
	for any $\D$, if it can be fooled by {\it some} \creative algorithm then it can also be fooled by a \creative algorithm 
	which is in fact a $\gam$-SDG with a private discriminator.
	%\shay{give a pointer to the proof where this is discussed more formally}.
}

\paragraph{Towards a converse of \cref{thm:main2}.}
By the above it follows that the family of classes $\D$ that can be fooled by a \creative algorithm 
	is contained in the family of all $\gam$--Foolable classes; 
	specifically, those which admit a $\gam$-SDG with a differentially private discriminator.

We do not know whether the converse holds; i.e.\ whether ``$\gam$--Foolability $\implies$ $\creative$-- Foolability''. %Nevertheless, 
%\roi{\Cref{thm:main3} can be considered as a weak converse for \cref{thm:main2}, and it shows that a slight strengthening of $\gam$s foolability to PAP-PAC learnability already implies DP-originiality (it has been shown (see \cite{alon} and also \cite{Bun15thresholds}) that Private PAC learnability implies finite Littlestone dimension and in turn $\gam$s foolability). This naturally lead to the following open problem:}
Nevertheless, the implication ``PAP-PAC learnability $\implies$ $\creative$--Foolability'' (\Cref{thm:main3}) can be regarded as an intermediate step towards this converse.
	Indeed, as discussed above, PAP-PAC learnablity implies $\gam$--Foolablility.
%	\strike{this follows from the fact that PAP-PAC learnability implies a finite Littlestone dimension (\cite{alon} and \cite{Bun15thresholds}) and from  \ref{thm:main1}
%	which yields an equivalence between $\gam$--Foolability and having a finite Littlestone dimension.} 	
	It is therefore natural to consider the following question, which is equivalent\footnote{I.e.\ an affirmative answer to \cref{q:pappacls} 
	is equivalent to the converse of \cref{thm:main2}.} 
	to the converse of \cref{thm:main2}:
	\begin{question}\label{q:pappacls}
	Let $\D$ be a class that has finite Littlestone dimension. Is $\D$ properly and privately learnable in the agnostic PAC setting?
	\end{question}
	A weaker form of this question -- Whether every Littlestone class is privately PAC Learnable? -- was posed by \cite{alon} as an open question (and was recently resolved in \cite{bun2020equivalence}).

\section{Discussion}\label{sec:discussion}
In this work we develop a theory for two types of constrained-SDG, sequential and private. Let us now discuss SDGs more generally, and we broadly want to consider algorithms that observe data, sampled from some real-life distribution, and in turn generate new synthetic examples that \emph{resemble} real-life samples, without any a-priori constraints. For example, consider an algorithm that receives as input some tunes from a specific music genre (e.g.\ jazz, rock, pop) and then outputs a new tune.  

Recently, there has been a remarkable breakthrough in the the construction of such SDGs with the introduction of the algorithmic frameworks of \emph{Generative Adversarial Networks} (GANs) \cite{goodfellow2014generative, goodfellow2016nips}, as well as Variational AutoEncoders (VAE) \cite{vae1,vae2}. In turn, the use of SDGs has seen many potential applications \cite{app1,app2,app3}. Here we follow a common interpretation of SDGs as  \emph{IPM minimizers} \citep{sanjeev1, bai2018approximability}. However, it was also observed \cite{sanjeev1, sanjeev2} that there is a critical gap between the task of generating \emph{new} synthetic data (such as new tunes) and the $\mmd$ minimization problem: In detail, \cref{obs:vc} shows that the $\mmd$ framework allows certain ``bad" solutions that \emph{memorize}. Specifically, let $S$ be a sufficiently large independent sample from the target distribution and consider the \emph{empirical distribution} as a candidate solution to the $\mmd$ minimization problem. Then, with high probability, the $\mmd$ distance between the empirical and the target distribution vanishes as $\lvert S\rvert$ grows.

To illustrate the problem, imagine that our goal is to generate new jazz tunes. 
Let us consider the discriminating class of all human music experts.
The solution suggested above uses the empirical distribution and simply ``generates" a tune from the training set\footnote{There are at most $7\cdot 10^9$ music experts in the world. Hence, by standard concentration inequalities a sample of size roughly $\frac{9}{\epsilon^2}\log 10 $ suffices to achieve $\mmd$ distance at most $\eps$ with high probability.}.
This clearly misses the goal of generating new and original tunes but the $\mmd$ distance minimization framework does not discard this solution. For this reason we often invoke further restrictions on the SDG and consider constrained-SDGs. For example, \cite{bai2018approximability} suggests to restrict the class of possible outputs $p_{syn}$ and shows that, under certain assumptions on the distribution $p_{real}$, the right choice of class $\D$ leads to learning the true underlying distribution (in Wasserstein distance).

In this work we explored two other types of constrained-SDGs, DP--SDGs and Sequential--SDGs, and we characterized the foolable classes in a distribution independent model, i.e. without making assumptions on the distribution $p_{real}$. One motivation for studying these models, as well as the  interest in a distribution independent setting, is the following underlying question:

\begin{center}
The output of Synthetic Data Generators should be \textbf{new} examples. But in what sense we require the output to be novel or \emph{distinct} from the training set? How and in what sense we should avoid copying the training data or even outputting a memorized version of it? 
\end{center}

To answer such questions is of practical importance. For example, consider a company that wishes to automatically generate music or images to be used commercially. One approach could be to train an SDG, and then sell the generated output. What can we say about the output of SDGs in this context? Are the images generated by the SDG original? Are they copying the data? or breaching copyright?

In this context, the differentially private setup comes with a very attractive interpretation that provides further motivation to study DP-SDGs, beyond preserving privacy of the dataset. To illustrate our interpretation of differential privacy as a criterion for originality consider the following situation: 
imagine that Lisa is a learning painter. She has learned to paint by observing samples of painting, produced by a mentor painter Mona. After a learning process, she draws a new painting $L$. Mona agrees that this new painting is a valid work of art, but Mona claims the result is not an original painting but a mere copy of a painting, say $M$, produced by Mona.

How can Lisa argue that paint $L$ is not a plagiary?
The easiest argument would be that she had never observed $M$. However, this line of defence is not always realistic as she must observe \emph{some} paintings. Instead, we will argue using the following thought experiment:
{\it What if} Lisa never observed $M$? 
Might she still create $L$? 
If we could prove that this is the case, then one could argue similarly that $L$ is not a palgiary.
%
%We will assume then that Lisa indeed observed $M$. \new{We then ask, \emph{how would things be different if}} Lisa never observed $M$? Might she still have come up with a painting such as $L$? If this is the case, then one could claim that $L$ is by no means a copy of $M$.

The last argument is captured by the notion of \emph{differential privacy}. 
%For a generating algorithm to be differentially private there needs to be a (non-negligible) chance of her producing $L$, even if painting $M$ was not in her training-set -- this should hold for any painting in her training set\footnote{\new{More accurately, we require that for any plausible event $E$ over the r.v $L$, $E$ needs to remain plausible had she not observed $M$. This is also captured by our intuition when we think of the event $E$ as the motif that been copied (e.g. a certain use of colors, contrast or any other property that Mona might claim that Lisa copied)}} (we refer the reader to \cref{sec:preliminaries} for a more formal introduction to the notion of DP).
In a nutshell, a randomized algorithm that receives a sequence of data points $\bar x$ as input is differentially private if removing/replacing a single data point in its input, does not affect its output $y$ by much; more accurately, for any event $E$ over the output $y$ that has non-negligible probability on input $\bar x$, then the probability remains non-negligible even after modifying one data point in $\bar x$.

The sequential setting also comes with an appealing interpretation in this context. A remarkable property of existing SDGs (e.g. GANs), that potentially reduces the likeliness of memorization, 
is that the generator's access to the sample is masked. In more detail, the generator only has restricted access to the training set via feedback from a discriminator that observes real data vs. synthetic data. Thus, potentially, the generator may avoid degenerate solutions that memorize. Nevertheless, even though the generator is not given a direct access to the training data, it could still be that information about this data could "leak" through the feedback it receives from the discriminator. This raises the question of whether $\gam$--Foolability can provide guarantees against memorization, and perhaps more importantly, in what sense? To start answering this question part of this work aims to understand the interconnection between the task of $\gam$-Fooling and the task of $\creative$--Fooling.

Finally, the above questions also motivate our interest in a distribution-independent setting, that avoids assumptions on the distribution $p_{real}$ which we often don't know. In detail, if we only cared about the resemblence between $p_{real}$ and $p_{syn}$ then we may be content with any algorithm that performs well in practice regardless of whether certain assumptions that we made in the analysis hold or not. But, if we care to obtain guarantees against copying or memorizing, then these should principally hold. And thus we should prefer to obtain our guarantees without too strong assumptions on the distribution $p_{real}$.
\section{ Technical overview: Littlestone classes are Sequentially--Foolable}\label{sec:uboverview}
As discussed, the main technical challenge in our proofs is to show that Littlestone classes are sequentially foolable. In this section we present the generator's strategy which is used in the proof of \cref{thm:main1quant}, \cref{it:ub}  
	to fool a class $\D$ with dual Littlestone dimension $\ell^*$.
	We will assume that $\D$ is symmetric (i.e.\ that $\D=1-\D$). 
	This assumption does not affect generality since one can symmetrize $\D$ by adding to it all functions in $1-\D$.
	This modification does not change the dual Littlestone dimension nor the associated $\gam$ game.

%	 the dual Littlestone dimension of $D\cup\{1-d : d\in \D\}$ 
%	is the same like the dual Littlestone dimension of $\D$.
%To simplify the presentation we focus on the finite case and assume that $\D$ is a finite symmetric class and denote its dual Littlestone dimension by $\ell^*$. 
%	The extension of the analysis to infinite classes is given in \cref{sec:infinite}.

The generator uses an online learner $\A$ for the dual class of $\Delta(\D)\subseteq [0,1]^\X$ whose existence is proved in \Cref{cor:bendavid}, and we refer the reader to \cref{sec:background} for further background in online learning as well as the exact statements. In a nutshell, $\A$ receives, sequentially, labelled examples, $(\bar{d}_t,y_t)$, from the domain $\Delta(\D)\times \{0,1\}$ and returns at each step $t$ a predictor $\hat{f}_t$ of the type $\hat{f}_t(\bar{d}) = \EE_{d\sim \bar{d}}\left[f_t(d)\right]$ for some function $f_t$ over the domain~$
\D$. Moreover $\A$ has the following guarantee over what we define as its regret: $$\mathsf{REGRET}_T(\A):=\sum_{t=1}^T |\hat{f}_t(\bar{d}_t) -y_t|-\min_{x\in \X}\sum_{t=1}^T| \EE_{d\sim \bar{d}_t} \left[d(x)\right]-y_t|\le \sqrt{\frac{1}{2}\ell^* T \log T},$$  

As stated, existence of $\A$ follows from a standard result on the existence of online learners for Littlestone classes. The regret term above can be seen as the standard regret term for an online learner, once we identify $\X$ as an hypothesis class over $\Delta(\D)$ and $x\in\X$ acts on~$\Delta(\D)$ by $\bar d \to \EE_{d\sim \bar d} [x(d)]$.

%Each distribution $p\in \Delta$

% The strategy we depict assumes an improper learner
% that receives labelled examples from the domain $\Delta(\D)=\Delta(\D)$ (i.e.\ the class of distributions over the discriminating class $\D$) and considers the 
% class $\X$ as a hypothesis class over $\D$: Note that we have switched roles, and we think of $\X$ as acting over $\D$ and extends this operation to elements in $\Delta(\D)$. Since we assume $\D$ has dual Littlestone dimension $\ell^*$ then by \cref{lem:dualLD} $\X$ has Littlestone dimension $\ell^*$ as a hypothesis class over $\D$ and by \cref{cor:bendavid} there exists an improper online learner as required.

\paragraph{Proof overview of \cref{thm:main1quant}, \cref{it:ub}.}
%To explain our approach, we begin by considering a simpler setting where it is assumed that 
%	the learner $\A$ is {\it proper with respect to $\Delta(\X)$} (or {\it $\Delta(\X)$-proper}) in the sense that at each iteration~$t$ 
%	it uses a predictor $\hat f_t = \EE_p [f_t]$,  where $f_t$ is a convex combination of hypotheses in $\X$:
%	%In other words, at each iteration~$t$ it uses a predictor $f_t$ of the form
%	\begin{align}\label{eq:proper} 
%	f_t(d) = p_t(d) = \EE_{x\sim p_t}[x(d)] ,
%	\end{align}
%	for some $p_t\in\Delta(\X)$. Let us denote this learner by $\A_{proper}$.  
%	Let us note in passing that this assumption can in fact be met 
%	if we consider only finite domains and allow the regret to scale with the size of the domain instead of the Littlestone dimension. 
%	For this setting we can, concretely, choose our online learner to be a Weighted majority algorithm \cite{littlestone1994weighted}. Similar uses of Weighted majority have  
%	been applied to generate synthetic data as well as fo other tasks with similar flavor \cite{hsu2013differential, hardt2010multiplicative}.

We begin by considering a simpler setting where it is assumed that 
	the learner $\A$ is {\it $\Delta(\X)$-proper} in the sense that at each iteration~$t$ 
	it uses a predictor~$\hat f_t = \EE_p [f_t]$,  where $f_t$ is a weighted average of hypotheses in $\X$; namely,
	\begin{align}\label{eq:proper} 
	(\forall d\in\D) : f_t(d) = p_t(d) = \EE_{x\sim p_t}[x(d)] ,
	\end{align}
	for some $p_t\in\Delta(\X)$. Let us denote this learner by $\A_{\Delta}$.  
	
	This assumption greatly simplifies things, and, in fact, the main technical challenge we face is in removing it. In fact, in the case of a finite domain $\X$ \cref{eq:proper} indeed holds, if we allow the regret to scale with the size of the domain instead of the Littlestone dimension. In this case we can concretely choose our online learner to be a {\it Weighted majority algorithm}~\cite{littlestone1994weighted}, 
	which satisfies \cref{eq:proper}. As such, Weighted majority has indeed been applied to generate synthetic data over finite domains \cite{hsu2013differential, hardt2010multiplicative}.
	
Coming back to the proof overview, the crucial point is that if $f_t$ satisfies \cref{eq:proper} 
	then the generator can submit $p_t\in \Delta(\X)$ to the discriminator. 
	Specifically, the generator can use~$A_{\Delta}$ as follows:
	at each iteration $t$, submit $p_t$ to the discriminator;
	then, unless $p_t$ fools $\D$ and the generator wins,
	receive a discriminator $d_t$ and obtain $f_{t+1}$ by feeding the labelled example~{$(\delta_{d_t},1)$} to~$\A_{\Delta}$.
	
We claim that after at most $\tilde O(\frac{\ell^*}{\epsilon^2})$ iterations the generator outputs a distribution that fools~$\D$:
	indeed, if the algorithm continues for more than $T$ iterations then for each $t\leq T$,
	\[\preal(d_t)-\hat f_t({\delta_{d_t}}) = \preal(d_t)-p_t(d_t)\ge \epsilon.\]
	Therefore, such a $T$ must satisfy:
%	The following calculation, then, shows that  $T$ must be of order $O(\frac{\ell^*}{\epsilon^2})$:
\ignore{
\begin{align*}
    \epsilon\cdot T &\le \sum_{t=1}^T \preal(d_t)-p_t(d_t) \\
    &\le  \sum_{t=1}^T \preal(d_t)-\hat f_t(d_t) \\
    &= \sum_{t=1}^T |y_t-\hat f_t(d_t)| - |y_t-\preal(d_t)| \tag{since $y_t=1$}\\
    &= \sum_{t=1}^T |y_t-\hat f_t(d_t)| -\sum_{t=1}^{T}\bigl\lvert y_t-\EE_{x\sim \preal}[x(d_t)]\bigr\rvert\\
    &\le \sum_{t=1}^T |y_t-\hat f_t(d_t)| - \min_{x\in \X}\sum_{t=1}^T|y_t-x(d_t)|\\
    &= \mathrm{REGRET}_T(\A_{\Delta})\\
    &= \tilde O(\sqrt{\ell^* T}),
\end{align*}}

\begin{align*}
    \epsilon\cdot T &\le \sum_{t=1}^T \preal(d_t)-p_t(d_t) 
    \le  \sum_{t=1}^T \preal(d_t)-\hat f_t(d_t) 
    = \sum_{t=1}^T |y_t-\hat f_t(d_t)| - |y_t-\preal(d_t)|\\
    &= \sum_{t=1}^T |y_t-\hat f_t(d_t)| -\sum_{t=1}^{T}\bigl\lvert y_t-\EE_{x\sim \preal}[x(d_t)]\bigr\rvert
    \le \sum_{t=1}^T |y_t-\hat f_t(d_t)| - \min_{x\in \X}\sum_{t=1}^T|y_t-x(d_t)|\\
    &= \mathrm{REGRET}_T(\A_{\Delta})
    = \tilde O(\sqrt{\ell^* T}),
\end{align*}
Where the first equality is true since $y_t=1$. This implies that $T=\tilde O(\frac{\ell^*}{\epsilon^2})$ as required.

We next proceed to the general case. 
	The main challenge is that existing online classification algorithms (including the algorithm implied by \cref{cor:bendavid}) 
	are not necessarily $\Delta(\X)$-proper.
	To bypass it, we first observe that one can relax the requirement that $\A$ is $\Delta(\X)$-proper to the requirement that $\A$ is $\Delta(\X)$-dominated in the sense that
	\begin{equation}\label{eq:dominated} 
	(\forall d\in \D):~f_t(d)\le p_t(d),
	\end{equation}
	for some $p_t\in \Delta(\X)$.
	Indeed, the above calculation remains valid under this weaker assumption.
	With this definition in hand we employ the minimax theorem to identify the following win-win situation:
	we check whether the predictor $f_t$ which is provided by $\A$ is sufficiently close to satisfying \cref{eq:dominated}
	(see the condition in the ``If'' statement in \cref{it:if} of \cref{alg:main}) and proceed as follows:
%The trick is to generate, somehow, at each iteration a discriminator $\bar{d}_t\in \Delta(\D)$ against $f_t$ using access to the discriminator $\D$ that accepts distributions. Then, roughly, the argument and analysis are the same as in the proper case.
%The bypass it, we first check whether the predictor $f_t$, 
%	which is provided by $\A$ is ``close'' to a predictor $p_t$ that satisfies \cref{eq:proper},
%	and consider two cases:
\begin{itemize}
\item If $f_t$ is sufficiently close to satisfying \cref{eq:dominated} then continue like before: 
	 in this case $f_t(d) \le \EE_{x\sim p_t} [x({d})] + O(\eps)$ for every~$d\in\D$.
	The generator then submits $p_t$ to the discriminator and uses the discriminator $d_t$ 
	provided by the discriminator as before to feed $\A$ with the example~$(d_t,1)$.
	By a similar calculation like above, this yields an increase of $\Omega(\eps)$ to the regret of $\A$.
	This case is depicted in \cref{it:if} in \cref{alg:main}.
\item In the complementing case, 
	a minimax argument implies that there exists $\bar{d}_t\in \Delta(\D)$ that separates~$f_t$ from \emph{all} dual hypotheses $x\in\X$ (see \cref{lem:vonneuman} below): 
	\begin{equation*} \bigl(\forall x\in \X \bigr): \EE_{d \sim \bar{d}_t}\left[ f_t(d) \right] > \EE_{d \sim \bar{d}_t}\left[ x(d) \right] + \frac{\epsilon}{2}
	\end{equation*}
%	namely, $ f_t(\bar{d}_t) \ge \EE_{x\sim p} [x(\bar{d}_t)] + \Omega(\eps)$ for all $p\in\Delta(\X)$, and in particular, 
%	\[f_t(\bar{d}_t) \ge \EE_{x\sim \preal} [x(\bar{d}_t)] + \Omega(\eps).\]
	By linearity, a corollary of the above equation is that $\EE_{d \sim \bar{d}_t}\left[ f_t(d) \right] > \EE_{d \sim \bar{d}_t}\left[ \preal(d) \right] + \frac{\epsilon}{2}$. 
	
	We, Thus, interpret $\bar{d}_t$ as a distinguishing function, and provide it to the learner $\A$ with a label $y_t=0$
	and yield an increase of $\Omega(\eps)$ to its regret.
	Note that here the discriminator is not used to find $\bar{d}_t$.
	This case is depicted in \cref{it:else} in \cref{alg:main}
\end{itemize}
To summarize, in each of the two cases, the regret of $\A$ is increased by $\Omega(\eps)$.
	Therefore, by the bound on $\A$'s regret,  it follows that after at most $\tilde O(\ell^*/\eps^2)$ rounds,
	the generator finds a fooling distribution.

\begin{figure}[h]
\fbox{\parbox{\textwidth}{
\begin{itemize}
\item Let $\D$ be a symmetric class with $\ldim^*(\D)=\ell^*$, and let $\eps>0$ be the error parameter.\\
Pick $\A$ to be an online learner for the dual class $\X$ like in \cref{cor:bendavid}, 
and set 
\[T= \Bigl\lceil\frac{4\ell^*}{\eps^2}\log \frac{4\ell^*}{\eps^2}\Bigr\rceil = O\Bigl(\frac{\ell^*}{\eps^2}\log \frac{\ell^*}{\eps} \Bigr).\]
%
%Set $T=\logceil{2\ell^*/\eps^2}= \tilde{O}(\ell^*/\eps^2)$ and 
%as in \cref{cor:bendavid}, where the horizon $T$ is set to be the minimum integer satisfying $T\cdot\eps/2 > \mathrm{REGRET}_T(\A) = O(\sqrt{\ell^*T\log T})$.
%Note that $T=\tilde O(\ell^*/\eps^2)$.
\item Set $\hat f_1(\bar{d})=\EE_{d\sim \bar{d}}[f_1(d)]$ as the predictor of $\A$ at its initial state.
\item For $t=1,\ldots, T$
\begin{enumerate}
\item\label{it:if} \textbf{If} there exists $p_t\in \Delta(\X)$ such that 
\begin{equation}\label{eq:separate1} (\forall d\in \D):\EE_{x\sim p_t}[f_t(d)-x(d)] \le \frac{\epsilon}{2}, \end{equation}

\textbf{then}

\begin{itemize}
\item pick such a $p_t$ and submit it to the discriminator.
\begin{itemize}
\item If the discriminator replies with ``Win'' then output $p_t$.
\item Else, receive from the discriminator $d_t\in \D$ such that \begin{equation}\label{eq:separate2}\preal(d_t)- p_t(d_t) \ge   \epsilon\end{equation}
\item Set $\bar{d}_t=\delta_{d_t}$, and $y_t=1$.
\end{itemize}

\end{itemize}
\item\label{it:else} \textbf{Else} 

\begin{itemize}
\item Find $\bar{d}_t \in \Delta(\D)$ such that 
\[\bigl(\forall x\in \X \bigr): \EE_{d \sim \bar{d}_t}\left[ f_t(d) - x(d)\right] > \frac{\epsilon}{2}\] 
(if no such $\bar{d}_t$ exists then output {\it ``error''}).
\item Set $y_t=0$.
\item Submit $p_{t}=p_{t-1}$ to the discriminator
and proceed to item 3 below (i.e.\ here the generator sends a dummy distribution to the discriminator and ignores the answer).
\end{itemize}
\item 
Update $\A$ with the observation $(\bar{d}_t,y_t)$, receive $\hat f_{t+1}$, set $f_{t+1}$ such that $\hat f_{t+1}(\bar{d}) = \EE_{\bar{d}}[f_{t+1}(d)]$
(such $f_{t+1}$ exists by the assumed properties of $\A$ -- see \Cref{cor:bendavid}), and proceed to the next iteration.
\end{enumerate}
\item Output ``Lost'' (we will prove that this point is never reached).
\end{itemize}}}
\caption{A fooling strategy  for the generator with respect to a symmetric discriminating class $\D$.}\label{alg:main}
\end{figure}
\paragraph{The computational complexity of DP-Fooling}
An issue that naturally arises in our context is the question of \emph{computational-complexity}. The work here focused solely on the question of sample complexity and tries to characterize the classes that are information-theoretically foolable. As always, computational issues change the picture very much. We depict here a DP-Fooling algorithm given access to a private learner. We note that the procedure is inefficient, irrespective of the complexity of the private learning algorithm. First, it relies on an online learner which need not be efficient. Infact, recently \citet{bun2020computational} showed the existence of a class that is (pure)-differentially private learnable but admits no efficient online learning algorithm. Second, the procedure relies also on a \emph{rounding} procedure that, given a target function $f$ which satisfies \cref{eq:separate1}, returns a distinguishing function that satisfies \cref{eq:separate2}. The problem can be written as an LP, albeit prohibitive. This work, though, suggests a deep connection between sequential foolability and DP-foolability. It would be interesting to find out if efficient DP-foolability implies efficient sequential foolability or even efficient online learnability of the class -- a negative answer (namely a class that is DP-foolable but not efficiently online learnable) could yield, conceptually, new techniques for generating synthetic data. %ADD FOR FULL VERSION
%\section*{Broader Impact}
%There are no foreseen ethical or societal consequences for the research presented herein.
\section{Proofs}\label{sec:background}

\subsection{Prelimineries}\label{sec:preliminaries2}
In this section we review some of the basic notations we will use as well as discuss further some standard definitions and notions in differential privacy and online learning.

We continue here the convention of \cref{sec:preliminaries}, and in this section we let $\W$ denote the domain and $\H\subseteq\{0,1\}^W$
to denote the functions class.
%
%and a class of discriminating functions. We will heavily exploit duality, and we will sometime think of $\X$ as a hypothesis class over $\D$ and vice versa. Thus, we will use in this section $\W$ to denote a domain and $\H\subseteq \{0,1\}^{\W}$ to denote a hypothesis class.}
%to 
\subsubsection{Notations}\label{sec:basics}

For a finite\footnote{The same notation will be used for infinite classes also. However we will properly define the the measure space and $\sigma$-algebra at later sections when we extend the results to the infinite regime.} set $\W$, let $\Delta(\W)$ denote the space of probability measures over $\W$. 
	Note that $\W$ naturally embeds in $\Delta(\W)$ by identifying $w\in \W$
	with the Dirac measure $\delta_w$ supported on $w$.
	Therefore, every $f:\Delta(\W)\to\mathbb{R}$ induces a $\W\to\mathbb{R}$ function via this identification.
	In the other direction, every~$f:\W\to\mathbb{R}$ naturally extends to a linear\footnote{A function $g:\Delta(\W)\to\mathbb{R}$ is {\it linear}
	if $g\bigl(\alpha p_1 + (1-\alpha) p_2\bigr) = \alpha g(p_1) + (1-\alpha)g(p_2)$, for all $\alpha\in[0,1]$} 
	map $\hat f:\Delta(\W)\to \mathbb{R}$ which is defined by~$\hat f(p) = \EE_p[f]$	for every~$p\in \Delta(\W)$.

We will often deal with boolean functions $f:\W\to\{0,1\}$,
	and in some cases we will treat $f$ as the subset of $\W$ that it indicates.
	For example, given a distribution $p\in\Delta(\W)$  we will use $p(f)$ to denote the measure of the subset that $f$ indicates 
	(i.e.\ $p(f) = \Pr_{w\sim p }[f(w) = 1]$).
	Given a class of functions $F\subseteq \{0,1\}^\W$, 
	its {\it dual class} is a class of $F\to\{0,1\}$ functions,
	where each function in it is associated with $w\in \W$
	and acts on $F$ according to the rule $f\mapsto f(w)$.
	By a slight abuse of notation we will denote the dual class with $\W$ 
	and use $w(f)$ to denoted the function associated with $w$ (i.e.\ $w(f):=f(w)$ for every $f\in F$).
	%\shay{Should we use $\W$ throughout the preliminaries and only introduce $\X$ when we discuss the model? We should also note that we  only focus on statistical efficiency and ignore other computational aspects.}
	
	Given a sample $S=(w_1,\ldots,w_m)\in \W^m$,
	the {\it empirical distribution} induced by $S$ is the discrete distribution $p_S$ defined by
	$p_S(w)=\frac{1}{m}\sum_{i=1}^{m}1[w=w_i].$
	
\subsubsection{Basic properties of Differential Privacy}
We will use the following three basic properties of algorithmic privacy.

\begin{lemma}[Post-Processing (Lemma 2.1 in \cite{vadhan})]\label{lem:pp}
If $M:\W^m \to \Sigma$ is $(\alpha,\beta)$-differentially private and $F:\Sigma\to Z$ is any (possibly randomized) function, 
then $F \circ M : \W^m \to Z$ is $(\alpha,\beta)$-differentially private.
\end{lemma}

\begin{lemma}[Composition (Lemma 2.3 in \cite{vadhan})]\label{lem:composition}
Let $M_1,...,M_k:\W^m\to \Sigma$ be $(\alpha,\beta)$-differentially private algorithms, and define $M:\W^M\to \Sigma^k$ by
\[M(\Omega) = \bigl(M_1(\Omega),M_2(\Omega),\ldots,M_k(\Omega)\bigr).\]
Then, M is $(k\alpha, k\beta)$-differentially private.
\end{lemma}

\begin{lemma}[Privacy Amplification (Lemma 4.12 in \cite{Bun15thresholds})]\label{lem:amplification}
Let $\alpha\le 1$ and let $M$ be a $(\alpha,\beta)$-differentially private algorithm operating on 
databases of size $u$. For $v>2u$, construct an algorithm $M'$ that on input database $\Omega\in \W^v$ subsamples (with replacement) $u$ points from $\Omega$ and runs $M$ on the result. Then $M'$ is $(\tilde \alpha,\tilde \beta)$-differentially private for 
\[ \tilde \alpha=6\alpha u/v \quad \tilde\beta = \exp(6\alpha u/v)\frac{4u}{v}\beta.\]
\end{lemma}
We remark that the requirement $\alpha\le 1$ can be replaced by $\alpha \leq c$ for any constant $c$ at the expanse 
	of increasing the constant factors in the definitions of $\tilde \alpha$ $\tilde \beta$.
	This follows by the same argument that is used to prove~\cref{lem:amplification} in~\cite{Bun15thresholds}.

%  restrictive it is in fact with out loss of generality as long as we assume that $\alpha\in O(1)$. 
%	Namely, any $(\alpha,\beta)$-differentially private algorithm $M$ can be amplified to satisfy $(1,\beta)$-p and rivacy by subsampling $O(e^{-\alpha})$ points from a database
%	larger by a $O(e^{\alpha})$ factor (see for example, \cite{li2012sampling}). 
%%	In particular, and for our purposes, any learning algorithm with sample complexity $m(\epsilon,\delta)$ that is $(\alpha,\beta(m))$ private can be turned into a $(1,\beta(m)$--differentially private learning algorithm at the cost of an increase in sample complexity to $e^{\alpha}m(\epsilon,\delta)$.
%%	

\subsubsection{Littlestone Dimension and Online Learning}\label{sec:online}
We begin be recalling the basic notion of Littlestone dimension.

\paragraph{Littlestone Dimension}
The Littlestone dimension is a combinatorial parameter that characterizes regret bounds in online learning, but also have recently been related to other concepts in machine learning such as differentially private learning~\cite{alon}. Perhaps surprisingly, the notion also plays a central role in Model Theory (\cite{shelah1990classification, chase2018model}, and see \cite{alon} for further discussion).

The definition of this parameter uses the notion of \emph{mistake-trees}:
these are binary decision trees whose internal nodes are labelled by elements of $\W$.
Any root-to-leaf path in a mistake tree can be described as a sequence of examples 
$(w_1,y_1),...,(w_d,y_d)$, where $w_i$ is the label of the $i$'th 
internal node in the path, and $y_i=+1$ if the $(i+1)$'th node  
in the path is the right child of the $i$'th node, and otherwise $y_i = 0$.
We say that a tree $T$ is \emph{shattered} by $\cH$ if for any root-to-leaf path
$(w_1,y_1),...,(w_d,y_d)$ in $T$ there is $h\in \cH$ such that $h(w_i)=y_i$, for all $i\leq d$.

The Littlestone dimension of $\cH$, denoted by $\ldim(\cH)$, is the maximum depth of a
complete tree that is shattered by~$\cH$.

The \emph{dual Littlestone Dimension} which we will denote by $\ldim^*(\cH)$ is the Littlestone dimension of the dual class (i.e. we consider $\W$ as the hypothesis class and $\H$ is the domain).
%domain $\X$
%when we think of the domain $\X$ as a hypothesis class over the $\cH$ when $x$ labels~$h$ according to the prediction rule: $x(h)=h(x)$. 
We will use the following fact:
\begin{lemma}\label{thm:little_dual}[Corollary 3.6 in~\cite{bhaskar2017thicket}]
Every class $\cH$ has a finite Littlestone dimension if and only if it has a finite dual Littlestone dimension. 
Moreover we have the following bound:
\[ \ldim^*(\cH)\le {2^{2^{\ldim(\cH)+2}}}-2\]
\end{lemma}

\paragraph{Online Learning} The Online learnability of Littlestone classes has been established by \cite{littlestone} in the realizable case and by \cite{bendavid} in the agnostic case. Ben-David et al's~\cite{bendavid} agnostic {\it Standard Online Algorithm} (SOA) will serve as a workhorse for our main results and we thus recall the online learning setting 
and state the relevant results. For a more exaustive survey on online learning we refer the reader to \cite{cesa,shai}.
  
In the a binary online setting we assume a domain $\W$ and a space of hypotheses $\cH\subseteq \{0,1\}^\W$. We consider the following \emph{oblivious} setting which  can be described as a repeated game between a learner $L$ and an adversary continuing for $T$ rounds; the {\it horizon} $T$ is fixed and known in advanced to both players. 
At the beginning of the game, the adversary picks a sequence of labelled examples $(w_t,y_t)_{t=1}^T\subseteq \W \times \{0,1\}$. 
Then, at each round $t\le T$, the learner chooses (perhaps randomly) a mapping $f_t: \W\to [0,1]$ and then gets to observe the labelled example $(w_t,y_t)$.
%In particular, the choice of concept $f_t$ may depend on previously seen examples $(x_1,y_1),\ldots (x_{t-1},y_{t-1})$. 
The performance of the learner $L$ is measured by her \emph{regret}, 
which is the difference between her loss and the loss of the best hypothesis in $\cH$: 
\begin{align}
\mathrm{REGRET}_T(L;\{w_t,y_t\}_{t=1}^T) =\sum_{t=1}^T \EE\left[|f_t(w_t)- y_t|\right] - \min_{h\in H} \sum |h(w_t)- y_t|,
\end{align}
where the expectation is taken over the randomness of the learner. 
Define 
\[\mathrm{REGRET}_T(L) = \sup_{\{w_t,y_t\}_{t=1}^T}\mathrm{REGRET}_T(L;\{w_t,y_t\}_{t=1}^T).\]
The following result establishes that Littlestone classes are learnable in this setting:
\begin{theorem}\label{thm:bendavid}[\cite{bendavid}]
Let $\cH$ be a class with Littlestone dimension $\ell$ and let $T$ be the horizon. 
Then, there exists an online learning algorithm $L$ such that 
\begin{align*}
\mathrm{REGRET}_T(L)\le  \sqrt{\frac{1}{2}\ell\cdot T\log T}
\end{align*}
\end{theorem}
We will need the following corollary of \cref{thm:bendavid}. 
Recall that $\Delta(\W)$ denotes the class of distributions over $\W$, and that every $f:\W\to [0,1]$ extends linearly to $\Delta(\W)$ by $\hat f(p) = \EE_{w\sim p}[f(w)]$.
The next statement concerns an online setting where the labelled example are of the form $(p_t,y_t)\in\Delta(\W)\times\{0,1\}$, and the regret of a learner $L$ with respect to  $\H\subseteq\{0,1\}^\W$ is defined by replacing each $h$ by its linear extension $\hat h$:
\begin{align*}
\mathrm{REGRET}_T(L;\{p_t,y_t\}_{t=1}^T) &=\sum_{t=1}^T \EE\left[|f_t(p_t)- y_t|\right] - \min_{h\in H} \sum |\hat h(p_t)- y_t|\\
&=\sum_{t=1}^T \EE\left[|f_t(p_t)- y_t|\right] - \min_{h\in H} \sum |\EE_{x\sim p_t} [h(w)]- y_t|\\
\end{align*}

\begin{corollary}\label{cor:bendavid}
Let $\cH$ be a finite class with Littlestone dimension $\ell$ and let $T$ be the horizon.
Then, there exists a deterministic online learner $L$ that receives labelled examples from the domain $\Delta(\W)$ 
such that
\[ \regret(L) \le \sqrt{\frac{1}{2} \ell T\log T}\]
Moreover,  at each iteration $t$ the predictor used by $L$ is of the form $\hat f_t(p) = \EE_{w\sim p}[f_t(w)]$, where $f_t$ is some~$\W\to [0,1]$ function.
%Moreover, for each iteration $t$ if $f_t$ is the predictor outputted by $L$, then $f_t$ is linear and $f_t(p)=\EE_{x\sim p} \left[f_t(x)\right]$ for every $p\in \Delta(\X)$.
%such that for every sequence $\{(p_t,y_t)\}_{t=1}^T \in \left(\Delta(\X)\times \{0,1\}\right)^T$,
%\[ \sum_{t=1}^T|f_t(p_t)-y_t| -\min_{h\in \cH}\sum_{t=1}^T|h(p_t)- y_t|=O(\sqrt{\ell T\log T})\] 
\end{corollary}
\cref{cor:bendavid} follows from \cref{thm:bendavid}; see \cref{sec:bendavid} for a proof.

\subsection{Proof of \cref{thm:main1quant}}
\subsubsection{Upper Bound: Proof of \cref{it:ub}}\label{prf:main_up}

In this section we prove the upper bound presented in \cref{thm:main1quant} in the case where $\X$ is finite (and in turn, $\D\subseteq \{0,1\}^X$ is also finite). As discussed though, the bounds will be independent of the domain size. The general case is proven in a similar fashion but is somewhat more delicate. The general proof is then given in \cref{sec:infinite}.

First note that we may assume without loss of generality that $\D$ is symmetric. 
	Indeed, if $\D$ is not symmetric then we may replace $\D$ with $\D\cup(1-\D)$, noting that this
	does not affect the $\gam$ game, namely (i) $\mmd_{\D} = \mmd_{\D\cup(1-\D)}$ (and so the goal of the generator remains the same),
	and (ii) the set of distinguishers the discriminator may use remains the same (recall that the discriminator is allowed to use distinguishers from $1-\D$).
	Also, one can verify that this modification does not change the dual Lttlestone dimension (i.e.\ $\ldim^*(\D)=\ldim^*(\D\cup (1-\D))$).
	
Therefore, we assume $\D$ is a finite symmetric class with dual Littlestone dimension $\ell^*$. 
	The generator used in the proof is depicted in \cref{alg:main}. 
	The generator uses an online learner~$\A$ for the dual class $\X$ with domain $\Delta(\D)$ as in \cref{cor:bendavid},
	where the horizon is set to be $T=\bigl\lceil\frac{4\ell^*}{\eps^2}\log \frac{4\ell^*}{\eps^2}\bigr\rceil$. 
	Let $D$ be an arbitrary discriminator, let $\preal\in\Delta(\X)$ be the target distribution, and let $\eps>0$ be the error parameter.  
	The proof follows from the next lemma: 
	\begin{lemma}\label{lem:vonneuman}
	Let $\D$ be a finite set of discriminators, let $f:\D\to[0,1]$, 
	Assume that,
	\[\bigl(\forall p\in \Delta(\X)\bigr)(\exists d\in \D):\EE_{x\sim p}[f(d)- x(d)]) > {\epsilon/2}.\] 
	Then: 
	\[\bigl(\exists \bar{d}\in \Delta(\D)\bigr)\bigl(\forall x\in \X\bigr):\EE_{d\sim \bar d}\left[ f(d)-x(d)\right] > {\epsilon/2}.\]
	\end{lemma}

Before proving this lemma, we show how it implies the desired upper bound on the round complexity. 	
	We first argue that the algorithm never outputs ``error'':
	indeed, since $\A$ only uses predictors of the form $\hat f_t (\bar d) = \EE_{\bar d}[f_t]$,
	\cref{lem:vonneuman} implies that whenever \cref{it:else} in the ``For'' loop is reached 
	then an appropriate $\bar{d}_t\in\Delta(\D)$ exists and therefore the  algorithm never outputs ``error''.

Next, we bound the number of rounds: let $T'\leq T$ be the number of iterations performed when the generator $G$
	runs against the discriminator $D$. The only way for the generator to lose is if the ``For'' loop ends without its winning and $T'=T$  . 
	Thus, It suffices to show that $T' < T$. The argument proceeds by showing that the regret of $\A$ in each iteration $t\leq T'$ increases by at least~$\eps/2$.
	This, combined with the bound on $\A$'s regret (from \cref{cor:bendavid}) will yield the desired bound.

We begin by analyzing the increase in $\A$'s regret. Let  $(\bar{d}_1,y_1),\ldots, (\bar{d}_{T'},y_{T'})$ and $\hat f_1,\ldots, \hat f_{T'}$
	be the sequences obtained during the execution of the algorithm as defined in \cref{alg:main}.
	Recall from \cref{cor:bendavid} that $\hat{f}_t(\bar{d})=\EE_{d\sim \bar d}[f_t(d)]$, where $f_t:\D\to[0,1]$.
	We claim that the following holds:
	\begin{align}\label{eq:improperD} 
	(\forall t\leq T'):
	\begin{cases}
	\EE_{d\sim \bar{d}_t}\bigl[\preal(d) - f_t(d)\bigr]  \ge \frac{\epsilon}{2} & \text{if } y_t=1,
	\\
	\EE_{d\sim \bar{d}_t}\bigl[f_t(d) - \preal(d)\bigr] \ge\frac{\epsilon}{2} & \text{if }y_t=0.
	\end{cases}
	\end{align}
	Indeed, if $y_t=1$ then by \cref{alg:main}, the chosen $p_t$ satisfies
	\[(\forall d\in \D): f_t(d) -\EE_{x\sim p_t}[x(d)]\le \frac{\epsilon}{2}.\]
	Since the discriminator replies with $d_t$ such that $\preal(d_t)- p_t(d_t)\ge \epsilon$, and $\bar{d}_t=\delta_{d_t}$, it follows that
	\begin{align*} 
	\EE_{d\sim \bar{d}_t}\bigl[\preal(d) - f_t(d)\bigr] &=\EE_{d\sim \bar{d}_t}[\preal(d_t)] - \EE_{d\sim \bar d_t} [f_t(d_t)]\\
										&=\preal(d_t)-f_t(d_t) \tag{because $\bar{d}_t = \delta_{d_t}$}\\	
										 &\ge  \EE_{x\sim \preal}[x(d_t)] - \left( \EE_{x\sim p_t}[x(d_t)]+ {\eps}/{2}\right)\\
										 &= \preal(d_t) - \left(p_t(d_t)+\eps/2\right) \\
										 &\ge \frac{\eps}{2},\end{align*}
	which is the first case in \cref{eq:improperD}. 
	Next consider the case when $y_t=0$.
	Since the algorithm never outputs ``error'', \cref{alg:main} implies that:
	\[\bigl(\forall x\in \X \bigr): \hat f_t(\bar{d}_t) - \EE_{d\sim \bar d_t}[x(d)] > \frac{\epsilon}{2}.\]
	Therefore, by linearity of expectation, $\EE_{d\sim \bar{d}_t}\bigl[f_t(d) - \preal(d) \bigr] = \hat f_t(\bar{d}_t) - \EE_{d\sim \bar d_t}[\preal(d)]\geq \frac{\epsilon}{2}$,
	which amounts to the second case in \cref{eq:improperD}.
	
We are now ready to conclude the proof by showing that $T' < T$.
Assume towards contradiction that $T' =  T$. Therefore, by \cref{eq:improperD}:
	\begin{align*}
	T\frac{\epsilon}{2} &\le \sum_{t=1}^T \bigl\lvert\EE_{d\sim \bar{d}_t}\bigl[\preal(d) - f_t(d)\bigr]\bigr\rvert
	\\&= \sum_{t=1}^T \bigl\lvert y_t-\EE_{d\sim \bar{d}_t}[f_t(d)]\bigr\rvert - \bigl\lvert y_t-\EE_{d\sim \bar d_t}[\preal(d_t)]\bigr\rvert 
	\tag{$y_t=1 \iff \EE_{d\sim \bar d_t}[\preal ( d_t)]\geq \EE_{d\sim \bar{d}_t}[f_t(d)]$}
	\\&= \sum_{t=1}^T \bigl\lvert y_t-\hat f_t(\bar d_t)\bigr\rvert - \EE_{x\sim \preal} \left[\bigl\lvert y_t-\EE_{d\sim d_t} x(d_t)\bigr\rvert\right]
	\\&\le \sum_{t=1}^T |y_t-f_t(\bar d_t)| - \min_{x\in \X} |y_t-\EE_{d\sim d_t} [x(d)]|
	\\&\le \mathrm{REGRET}_T(\A).
	\\& \le \sqrt{\frac{1}{2} \ell^* T\log T}
	\end{align*}
	Thus, we obtain that $\frac{T}{\log T} \le \frac{2\ell^*}{\epsilon^2}$, 
	however our choice of $T=\bigl\lceil\frac{4\ell^*}{\eps^2}\log \frac{4\ell^*}{\eps^2}\bigr\rceil$ ensures that this is impossible. 
	Indeed:
	\begin{align*}
	\frac{T}{\log T} &\geq\frac{\frac{4\ell^*}{\epsilon^2}\log\frac{4\ell^*}{\epsilon^2}}{\log \frac{4\ell^*}{\epsilon^2}+\log\log\frac{4\ell^*}{\epsilon^2}}
	\\&=\frac{\frac{4\ell^*}{\epsilon^2}}{1+\frac{\log\log\frac{4\ell^*}{\epsilon^2}}{ \log\frac{4\ell^*}{\epsilon^2}}}
	\\& > \frac{\frac{4\ell^*}{\epsilon^2}}{2}
	\\& = \frac{2\ell^*}{\epsilon^2}.
	\end{align*} 
This finishes the proof of \cref{it:ub}.
%\qed

We end this section by proving \cref{lem:vonneuman}.

\begin{proof}[\textbf{Proof of \cref{lem:vonneuman}}]
The proof hinges on Von Neuman's Minimax Theorem. Let $D,f$ as in the formulation of the theorem, 
	and consider the following zero-sum game: the pure strategies of the maximizer are indexed by $d\in \D$, 
	the pure strategies of the minimizer are indexed by $x\in X$, and the payoff (for pure strategies) is defined by $m(d,x)= f(d) - x(d)$. 
	Note that the payoff function for mixed strategies~$\bar d\in \Delta(\D), p\in\Delta(\X)$ satisfies 
	\[m(\bar d,p) = \EE_{x\sim p}[\hat f(\bar d) - \EE_{d\sim \bar d}x(d)] =  \EE_{d\sim \bar d}\bigl[f(d)- \EE_{x\sim p}[x(d)]\bigr].\]

We next apply Von Neuman's Minimax Theorem on this game (Here we use the assumption that $\X$ and, in turn, $\D$ are finite). The premise of the lemma amounts to 
	\[\min_{p\in\Delta(\X)}\max_{d\in \D} m(d,p) > \eps/2.\]
	Therefore, by the Minimax Theorem also 
	\[\max_{\bar d\in \Delta(\D)}\min_{x\in \X} m(\bar d,x) > \eps/2,\]
	which amounts to the conclusion of the lemma.
\end{proof}

\paragraph{A remark.}
A natural variant of the $\gam$ setting follows by letting the discriminator $D$
	to adaptively change the target distribution $\preal$ as the game proceeds 
	($D$ would still be required to maintain the existence of a distribution $\preal$ which is consistent with all of its answers).
	This modification allows for stronger discriminators and therefore, potentially, for a more restrictive notion of $\gam$--Foolability.
	However, the above proof extends to this setting verbatim.

\subsubsection{Lower Bound: Proof of \cref{it:lb}}
Let $\D$ be a class as in the theorem statement, let $G$ be a generator for $\D$, and let~$\eps < \frac{1}{2}$.
	We will construct a discriminator $D$ and a target distribution $\preal$ such that $G$ requires at least~$\frac{\ell^*}{2}$
	rounds in order to find $p$ such that $\mmd_\D(p,\preal)\leq \eps$. 
	
To this end, pick a shattered mistake-tree $\T$ of depth $\ell^*$ whose internal nodes are labelled by elements of $\D$
	and whose leaves are labelled by elements of $\X$. 	%\roi{w.l.o.g we will assume that $\D$ consists exactly of all the labels of the node in the mistake tree $\T$. Clearly, a discriminator that restricts its strategies to output functions from labels in the mistake--tree is only weaker than any discriminator whose strategies are restricted to $\D$.}

\paragraph{The discriminator.}
The target distribution will be a Dirac distribution $\delta_x$ where $x$ is one of the labels of $\T$'s leaves.	
	We will use the following discriminator $D$ which is defined whenever~$\preal$ is one of these distributions:
	assume that $\preal = \delta_x$, and consider all functions in~$\D$ that label the path from the root towards the leaf whose label is $x$,
	\[d_1,d_2,\ldots, d_{\ell^*}.\] 
	Let $p_1$ be the distribution the generator submitted in the first round.
	Then the discriminator picks the first $i$ such that $\lvert p_t(d_1) -\preal(d_1)\rvert > \eps$,
	and sends the generator either $d_i$ or $1-d_i$ according to the convention in \cref{eq:win}. If no such $d_i$ exists, the discriminator outputs $\win$.
	Similarly, at round $t$ let $i_{t-1}$ denote the index of the distinguisher sent in the previous round;
	then, the discriminator acts the same with the modification
	that it picks the first $i_{t-1}+1\leq i\leq \ell^* $ such that $\lvert p_t(d_i) -\preal(d_i)\rvert > \eps$.

\paragraph{Analysis.}
The following claim implies that for every generator $G$, there exists a distribution $\delta_x$
	such that if $\preal=\delta_x$ then the above discriminator $D$ forces $G$ to play at least $\ell^*/2$ rounds.
\begin{claim}
Let $G$ be a generator for $\D$.
	Pick $\preal$ uniformly at random from the set $\{\delta_x :  x \text{ labels a leaf in }\T\}$.
	Then the expected number of rounds in the $\gam$ game when $G$ is the generator and $D=D(\T)$
	is the discriminator is at least $\frac{\ell^*}{2}$.
\end{claim}
\begin{proof}
For every $i \leq \ell^*$, let $X_i$ denote the indicator of the event that the $i$'th function
	on the path towards the leaf corresponding to $\preal$ was used by $D$ as a distinguisher.
	Note that the number of rounds $X$ satisfies $X =  \sum_{i=1}^{\ell^*}X_i$.
	Thus, by linearity of expectation it suffices to argue that 
	\[\EE[X_i] = \Pr[X_i=1] \geq \frac{1}{2}.\]
	Consider $X_1$: let $p_1$ denote the first distribution submitted by $G$.
	Note that $X_1=1$ if
	\begin{itemize}
	\item[(i)] $p_1(d_1) \geq \frac{1}{2}$ and the leaf labelled $x$ belongs to the left subtree from the root, or 
	\item[(ii)] $p_1(d_1) < \frac{1}{2}$ and the leaf labelled $x$ belongs to the right subtree from the root.
	\end{itemize}
	In either way $\Pr[X_1=1] \geq \frac{1}{2}$, since this leaf is drawn uniformly. 
	Similarly, for every conditioning on the values of $X_1,\ldots, X_{i-1}$
	we have $\Pr[X_i = 1 \vert X_1\ldots X_{i-1}]\geq \frac{1}{2}$
	(follows from the same argument applied on subtrees corresponding to the conditioning).
	%(note that each such conditioning restricts the leaf labelled $x$ to be a uniform leaf in a corresponding subtree).
	This yields that $\EE[X_i]=\Pr[X_i=1]\geq\frac{1}{2}$ for every~$i$ as required.

\end{proof}

\subsection{Proof of \cref{thm:main3}}\label{sec:main3}
\paragraph{Proof Roadmap.}
We will show the following entailments:
 \ref{it:pac}$\Rightarrow$\ref{it:foolable}$\Rightarrow$\ref{it:sanitizable}$\Rightarrow$ \ref{it:puc}$\Rightarrow$ \ref{it:pac}. 
	This will conclude the proof.

\paragraph{Overview of \ref{it:pac}$\Rightarrow$\ref{it:foolable}.}
We next overview the derivation of \ref{it:pac}$\Rightarrow$\ref{it:foolable} which is the most involved derivation.
	Let $\preal$ denote the target distribution we wish to fool.
	The argument relies on the following simple observation:
	 let $S$ be a sufficiently large independent sample from $\preal$.
	Then, it suffices to privately output a distribution $\syn$ such that~$\mmd_\D(\syn,p_S)\le \frac{\eps}{2}$, 
	where $p_S$ is the empirical distribution.
	Indeed, if $S$ is sufficiently large then by standard uniform convergence bounds: $\mmd_\D(p_S,\preal)\leq\frac{\eps}{2}$,
	which implies that $\mmd_\D(\syn,\preal)\leq \eps$ as required.
%The first observation we make for the proof is that, given an IID sample, S, drawn from a distribution $\preal$, it is enough to construct privately a distribution $\syn$ such that $\mmd(\syn,p_S)\le \eps$. Then, through standard argument, we show that for sufficiently large sample, w.h.p, $\mmd(\syn,\preal)=O(\eps)$.

The output distribution $\syn$ is constructed using a carefully tailored Sequential-SDG
	with a {\it private discriminator} $D$.
	That is, $D$'s input distribution is the empirical distribution $p_S$, 
	and for every submitted distribution $p_t$, it either replies with a discriminating function~$d_t$
	or with ``$\win$'' if no discriminating function exists.
	The crucial point is that it does so in a differentially private manner with respect to the input sample $S$.
	The existence of such a discriminator~$D$ follows via the assumed PAP-PAC learner. 

%The main idea of the proof is to construct a private discriminator from a given PAP-PAC learner. Namely, given an input sample $S$ we give the discriminator as input the empirical distribution $p_S$ over the input sample. For every candidate distribution $p_t$, the discriminator will invoke the PAP-PAC learner and output a discriminating function $d_t$ (or yield $\win$) privately.  

Once the private discriminator $D$ is constructed, we turn to find a generator $G$ with a bounded round complexity. 
	This follows from \cref{thm:main1quant} and a result by \cite{alon,Bun15thresholds}:
	by \cite{alon,Bun15thresholds} PAP-PAC learnability implies a finite Littlestone dimension,
	and therefore by \cref{thm:main1quant} there is a generator~$G$ with a bounded round complexity.
	The desired \creative fooling algorithm then follows by letting~$G$ and $D$ play against each other and outputting the final distribution that $G$ obtains.
	The privacy guarantee follows by the {\it composition lemma} (\cref{lem:composition}) which bounds the privacy leakage in terms
	of the number of rounds (which is bounded by the choice of $G$) and the privacy leakage per round (which is bounded by the choice of $D$).

%Once we construct such a discriminator, we let the discriminator play against a generator with bounded round complexity. Such a generator indeed exists because PAP-PAC learnability implies finite Littlestone dimension and in turn implies the existence of such a generator. The generator will then output, after $T$ interactions with the discriminator, a distribution $\syn$. Because the number of calls to the discriminator is bounded, by a composition Lemma and the privacy guarantees of the discriminator we obtain that the output distribution of the generator will indeed preserve privacy.

One difficulty that is handled in the proof arises because the discriminator is differentially private and because the PAP-PAC algorithm
	may err with some probability. Indeed, these prevent~$D$ from satisfying the requirements of a discriminator as defined in the $\gam$ setting.
	In particular,~$D$ cannot reply deterministically whether $\mmd_\D(p_S,p_t)<\epsilon$ as this could compromise privacy.
	Also, whenever the assumed PAP-PAC algorithm errs, $D$ may reply with an illegal distinguisher that does not satisfy \cref{eq:win}.
%	because the discriminator $D$ needs to be differentially private and because it relies on a PAC-algorithm which may fail with some probability, 
%	$D$ cannot fully satisfy the requirements of a discriminator as depicted in the setting of Generative Adversarial Machine. 
%	In particular, it cannot publish deterministically whether $\mmd(p_S,p_t)<\epsilon$, as this could compromise privacy. 
%	Also, if the PAP-PAC learning algorithm errs $D$ may reply with a distinguisher $d_t$ which do not satisfy \cref{eq:win}.

To overcome this difficulty we ensure that $D$ satisfies the following with high probability: 
	if $\mmd_\D(p_S,p_t)>\epsilon$ then $D$ outputs a legal $d_t$, and if $\mmd_\D(p_S,p_t)<\frac{\epsilon}{2}$ then it outputs $\win$ as required. 
	When $\frac{\eps}{2}\leq\mmd_\D(p_S,p_t)\leq \epsilon$ it may either output $\win$ or a legal discriminator $d_t$. 
	As we show in the proof, this behaviour of $D$ will not affect the correctness of the overall argument.
%	%change the fact that with high probability, $G$ will still find a fooling
%	distribution after a bounded number of rounds. 
%	
%	As we show in the proof, a generator with bounded round complexity (against a ``legal discriminator"), 
%	if ran against such a discriminator would still have bounded round complexity as long as the discriminator's output is valid (which happens w.h.p)
 
\paragraph{Proof of \cref{thm:main3}}
As discussed, the equivalence is proven by showing:
 \ref{it:pac}$\Rightarrow$\ref{it:foolable}$\Rightarrow$\ref{it:sanitizable}$\Rightarrow$\ref{it:puc}$\Rightarrow$\ref{it:pac}.

\paragraph{\ref{it:pac}$\Rightarrow$\ref{it:foolable}.} We next prove the following proposition, which proves that PAP-PAC learnability implies $DP$-Fooling.
\begin{proposition}\label{lem:quant:1}
Let $\D$ be a class that is privately PAC learnable with a $(\alpha(m),\beta(m))$--differentially-private learner $L$. 

Then, for every $0<\kappa<1$ there exists an $(O(\alpha(m^{1-\kappa})),e^{O\left(\alpha(m^{1-\kappa})\right)}\cdot O(\beta(m^{1-1/\kappa}))$-differentially-private  fooling algorithm for $\D$ with sample complexity

\begin{equation} S=\tilde O\left(\left(m\left(\epsilon/8,\delta/2\cdot \tilde O\left(\frac{\ell^*}{\epsilon^2}\right)\right)+\frac{64\log \frac{\tilde O\left(\frac{\ell^*}{\epsilon^2}\right)}{\delta}}{\epsilon}\right)\cdot \left(\frac{\ell^*}{\epsilon^2}\right)+\left(\frac{\ell^*}{\epsilon^2}\right)^{1/\kappa}\right),\end{equation}
where $\ell^*$ is the dual Littlestone dimension of $\D$.
\end{proposition}
Before we prove the proposition, note that if $\D$ is DP-PAC learnable, then we have that $\alpha(\sqrt{|S|})=O(1)$ and $\beta(\sqrt{|S|})$ is negligible. Therefore, plugging in $\kappa=1/2$ in \cref{lem:quant:2}, indeed the desired entailment follows.
\begin{proof}
Let $\preal$ denote the unknown target distribution and let $\eps_0,\delta_0$ be the error and confidence parameters. We assume that $L$ is a DP-learner with privacy parameters $(\alpha,\beta)$ where $\alpha=O(1)$ and $\beta$ is negligible in $m$.

	Draw independently from~$\preal$ a sufficiently large input sample $S$ of size~$|S|$ to be specified later.
	At this point we require $|S|$ to be large enough so that $\mmd_\D(\preal,p_S)\leq \frac{\eps_0}{2}$ with probability at least $1-\frac{\delta_0}{2}$. 
	By standard uniform convergence bounds (\cite{Vapnik71uniform}) it suffices to require
	\begin{equation}\label{eq:size1}
	\lvert S\rvert\geq \Omega\Bigl(\frac{\textrm{VC}(\D)+\log(1/\delta_0)}{\eps_0^2}\Bigr),	
	\end{equation}
	where $\textrm{VC}(\D)$ is the VC-dimension of $\D$ (observe that $\D$ must have a finite VC dimension as it is PAC learnable).
%	This lower bound on $|S|$ is set according to standard uniform convergence bounds  
	By the triangle inequality, this reduces our goal to privately output a distribution $\syn$ 
	so that $\mmd_\D(p_S,\syn)\leq\frac{\eps_0}{2}$ with probability~$1-\frac{\delta_0}{2}$
	(this will imply that $\mmd_\D(\preal,\syn)\leq \eps_0$ with probability $1-\delta_0$).
	
	 As explained in the proof outline, the latter task is achieved by a Sequential-SDG which we will next describe. Inorder to construct the desired $\gam$-SDG, we first observe that $\D$ is $\gam$--Foolable.
		Indeed, by \cref{thm:main1} it suffices to argue that $\D$ has a finite Littlestone dimension,
		which follows by~\cite{alon} since~$\D$ is privately learnable.
		
		Now,  pick a generator $G$ that fools $\D$ with round complexity $T(\eps)$ as in \cref{thm:main1quant},
		and pick a discriminator $D$ as in \cref{alg:discriminator}.
		Note that $D$ uses a PAP-PAC learner for the class
		$\D\cup(1-\D)$ whose existence follows from the PAP-PAC learnability of~$\D$ via standard arguments (which we omit).
		The next lemma summarizes the properties of $D$ that are needed for the proof.
%
%For the proof we rely on a discriminator which we now state its  main properties in \cref{lem:discriminator}. The discriminator is then depicted in \cref{alg:discriminator} and relies upon a procedure $\thrsh$ whose main properties are depicted in \cref{fig:thrsh}. The proof of \cref{lem:discriminator} is dettered to the end of this section. 
\begin{lemma}\label{lem:discriminator}
Let $D$ be the discriminator defined in \cref{alg:discriminator} 
	with input parameters~$(\epsilon,\delta,\tau)$ and input sample $S$, 
	and let $M$ be the assumed PAP-PAC learner for $\D\cup(1-\D)$ with sample complexity $m(\epsilon,\delta)$ 
	and privacy parameters $(\alpha,\beta)$.
	Then, $D$ is $\privacyexpression $-private,
	and if $S$ satisfies
	\begin{equation}\label{eq:size3}
        \lvert S\rvert\ge 
        \max\left(\frac{m(\epsilon/8,\tau\delta/2)}{\tau},\frac{64\log(\tau\delta/2)}{\eps\tau}\right)
%   \max\left(\frac{m(\epsilon/8,\tau\delta/2)}{\tau},1/\tau^{1.001},0.001\frac{\log 2/\delta}{\epsilon\tau}\right)= \Omega\left(\frac{m(\epsilon/8,\tau\delta/2)}{\tau}\right)
       \end{equation} 
	then the following holds with probability at least $(1-\tau\delta)$
     \begin{itemize}    
     	\item[(i)] If $D$ outputs $d_t$ then $p_S(d_t)- p_t(d_t)\ge \frac{\epsilon}{2}$.
        \item[(ii)] If $D$ outputs ``$\win$'' then $\mmd_\D(p_S,p_t)\le \epsilon$.
     \end{itemize}
%\begin{itemize}
%    \item \underline{\textrm{Privacy}:}     
%    \item \underline{\textrm{Utility}:} The 
%\end{itemize}
\end{lemma}
%	Let $M$ be a PAP-PAC learner for the class $\D$. 
%	As $\D$ is privately learnable, by \cite{alon} it has a finite Littlestone dimension\footnote{
%	in fact, we assume that $\D$ is \emph{properly} privately learnable, 
%	hence the result is already a corollary of the lower bound for proper private learning by \cite{Bun15thresholds}, combined with Theorem 3 in \cite{alon}}.
%%	and 	by \cref{thm:little_dual} a finite dual Littlestone dimension\footnote{in fact, we assume that $\D$ is \emph{properly} privately learnable, 
%%	hence the result is also a corollary of \cite{Bun15thresholds} by Theorem 3 in \cite{alon}}. 
%	Thus, \cref{thm:main1} and \cref{thm:main1quant} imply that $\D$ is $\gam$--Foolable by a generator $G$ that fools it
%	with round complexity $T(\epsilon)=\tilde O(\ell^*/\epsilon^2)$, where $\ell^*$  is its dual Littlestone dimension.
%	Pick $D$ be the discriminator depicted in \cref{alg:discriminator}.

We first use \cref{lem:discriminator} to conclude the proof of \ref{it:pac}$\Rightarrow$\ref{it:foolable}
	and then prove \cref{lem:discriminator}.
%The proof of \cref{lem:discriminator} is deterred to the end of this proof and we next use it to 
%	conclude the proof of \cref{thm:main3}.

The fooling algorithm we consider proceeds as follows.

\begin{itemize}
%\item Given a sample $S$ of size $m$ and error parameters $\epsilon_0,\delta_0$, set $T_0=\min(m^{0.99},T(\epsilon_0/2))$ and $\tau_0=1/T_0$.
\item Set $G$ to be a generator with round complexity $T(\epsilon)$ and set its error parameter to be $\frac{\epsilon_0}{2}$.
\item Set  the number of rounds $T_0=\min\{|S|^{\kappa},T(\epsilon_0/4)\}$, and let $\tau_0=1/T_0$.   %\shay{$T_0 = m^{0.99}$}
\item Set $D$ be the discriminator depicted in \cref{alg:discriminator} and set its parameters to be $(\eps,\delta,\tau) = (\frac{\epsilon_0}{2},\frac{\delta_0}{2},\tau_0)$ and its input sample to be $S$.
\item Let $G$ and $D$ play against each other for (at most) $T_0$ rounds. 
\item Output the final distribution which is held by $G$.
% $p_{out}$ the output distribution of the induced algorithm: namely $p_{out}=p_t$ if at some round $t$ the discriminator yields $\win$ and else $p_{out}=p_{T_0}$.
\end{itemize}
%We will assume throughout the proof that $|S|$ is sufficiently large and in particular the algorithm does not output $\top$, indeed the algorithm is clearly private for sample sizes for which the output is $\top$ and we only need to prove the fooling property for sufficiently large sample. 

We next prove the privacy and fooling properties as required by a \creative algorithm:
\paragraph{Privacy.}
	Note that since $G$ is deterministic then the output distribution~$p_{out}$ is completely determined by the sequence of discriminating functions $d_1,\ldots, d_{T'}$ 
	outputted by the discriminator.

For simplicity and without loss of generality we assume that $T'=T_0$: 
	indeed, if $T'<T_0$ then extend it by repeating the last discriminating function;
	this does not change the fact that $p_{out}$ is determined by the sequence $d_1,\ldots, d_{T'},\ldots d_{T_0}$.
 
 Recall that by \cref{lem:discriminator} $D$ is $(\left(6\tau_0 \alpha(\tau_0 |S|)+\tau_0\right),\left(4e^{6\tau_0\alpha(\tau_0 |S|)}\tau_0\beta(\tau_0 |S|)\right))$-private.
	Therefore, since the number of rounds in which $D$ is applied is $T_0$, 
	by {\it composition} (\cref{lem:composition}) and {\it post-processing} (\cref{lem:pp}) it follows that the entire algorithm is 
	\[\Bigl(T_0\left(6\tau_0 \alpha(\tau_0 |S|)+\tau_0\right),T_0\bigl(4e^{6\tau_0\alpha(\tau_0 |S|)}\tau_0\beta(\tau_0 |S|)\bigr)\Bigr)\mbox{-private.}\]
 	Our choices of $\tau_0=\frac{1}{T_0}$ and $T_0$ guarantee that $\tau_0> 1/|S|^\kappa$,
	and plugging it in yields privacy guarantee of $(6\alpha(|S|^{1-\kappa})+1, 4e^{O(|S|^{1-\kappa})}\beta(|S|^{1-\kappa})$.

\paragraph{Fooling.} 
First note that if $S$ satisfies \cref{eq:size3} with $(\eps,\delta,\tau):=(\eps_0,\frac{\delta_0}{2},\tau_0)$ then with probability at least $1-\frac{\delta_0}{2}$ the following holds:
	in every iteration $t\leq T_0$, either $p_S(d_t) - p_t(d_t)\ge \frac{\epsilon_0}{4}$, 
	or the discriminator yields $\win$ and $\mmd_\D(p_S,p_t)\le \frac{\eps_0}{2}$.
	This follows by a union bound via the utility guarantee in \cref{lem:discriminator}.
	Assuming this event holds, we claim that if $|S|$
	is set to satisfy~$|S|^{\kappa}\geq T(\frac{\eps_0}{4})$
	then the output distribution $\syn$ satisfies~$\mmd_\D(p_S,\syn)\leq \frac{\eps_0}{2}$.
	This follows since as long as the sequential game proceeds the generator suffers a loss of at least $\frac{\eps_0}{4}$ in every round,
	and  the number of rounds is set as, in this case, to be ~$T\bigl(\frac{\eps_0}{4}\bigr)$.
	Therefore we require 
	\begin{equation}\label{eq:size2}
	|S|^{\kappa} \geq T\bigl(\frac{\eps_0}{4}\bigr) = \Omega\Bigl(\frac{\ell^*}{\eps_0^2}\log\frac{\ell^*}{\eps_0}\Bigr).
	\end{equation}

To conclude, if $|S|$ is set to satisfy \cref{eq:size3,eq:size1,eq:size2} then 
	with probability at least $1-\delta_0$ both~$\mmd_\D(\preal,p_S)\leq \frac{\eps_0}{2}$ and~$\mmd_\D(p_S,\syn)\leq \frac{\eps_0}{2}$,
	which implies that $\mmd_\D(\preal,\syn)\leq\eps_0$ as required.
Finally, observe that for $\tau<1$, we have $m(\epsilon/8,\tau\delta/2)=\Omega\left(\frac{\textrm{VC}(\D)+\log 1/\delta}{\epsilon^2}\right)$. Indeed, any PAC learner (not necessarily private), in particular $L$, requires that many samples to learn. Taken together we obtain that
\[S= \tilde O \left(\left(m(\epsilon_0/8,\tau\delta_0/2)+\frac{64\log \tau/\delta}{\epsilon}\right)\cdot \frac{1}{\tau} +\left(\frac{\ell^*}{\epsilon_0^2}\right)^{1/\kappa}\right).\]
Plugging in $\tau=T(\epsilon/4)$ yields the desired result.

	This %yields a bound of $|S|\approx \frac{{(\ell^*\log(\ell^*/\eps_0))}^{1.01}}{\eps_0^{2.02}} + \frac{d + \log(1/\delta_0)}{\eps_0^2}$ on the sample complexity of DP-fooling~$\D$ and 
	concludes the proof of \cref{lem:quant:2} and in particular the entailment \ref{it:pac}$\Rightarrow$\ref{it:foolable}.
\end{proof}

\begin{proof}[\textbf{Proof of \cref{lem:discriminator}}]
%We next establish the properties of the discriminator depicted in \cref{alg:discriminator}. 
%	The discriminator receives as input a sample $S$ and a distribution $p_t$ 
%	and declares ``WIN'' if $\mmd(p_S,p_t)\leq \eps$ or outputs a discriminating function $d_t$  
%	such that $p_S(d_t)- p_t(d_t)\ge \frac{\epsilon}{2}$. 
%	In contrast with a standard discriminator we will require privacy guarantees. 
		
Let $S$ be the input sample, let $p_S$ denote the uniform distribution over $S$, and let $p_t$ denote the distribution submitted by the generator.
	The discriminator operates as follows (see \cref{alg:discriminator}): it feeds the assumed PAP-PAC learner a labeled sample $S_\ell=\{(x_i,y_i)\}$ 
	that is drawn from the following distribution $q_t$: 
	first the label $y_i$ is drawn uniformly from $\{0,1\}$;
	if $y_i=0$ then draw $x_i\sim p_S$ and if $y_i=1$ then draw~$x_i\sim p_t$. 
	Let $d_t$ denote the output of the PAP-PAC learner on the input sample $S$.
	Observe that the loss $L_{q_t}(\cdot)$ satisfies
	\begin{equation}\label{eq:lossvsadvantage}
	L_{q_t}(d) = \frac{p_S(d)+(1- p_t(d)) }{2}=\frac{1+p_S(d) - p_t(d) }{2}.
	\end{equation}
	Next, the discriminator checks whether $p_S(d_t) - p_t(d_t)> \frac{\eps}{2}$ (equivalently, if $L_{q_t}(d_t) < \frac{1-\eps/2}{2}$), 
	and sends $d_t$ the generator if so, and reply with ``WIN'' otherwise.
	The issue is that checking this "If" condition naivly may violate privacy,
	and in order to avoid it we add noise to this check by a mechanism from \cite{dwork1} (see \cref{fig:thrsh}):
	roughly, this mechanism receives a data set of scalars $\Sigma =\{\sigma_i\}_{i=1}^m$, a threshold parameter $c$ and a margin parameters $N$,
	and outputs $\top$ if $\sum_{i=1}^m \sigma_i > c+ O(1/N)$ or $\perp$ if $\sum_{i=1}^m \sigma_i < c- O(1/N)$. 
	The distinguisher applies this mechanism over the sequence of scalars~$\{d_t(x_1),\ldots, d_t(x_m)\}$.

%If we did not care for privacy, the discriminator could at this point simply estimate $p_S(d_t)$ and output $d_t$ if $p_S(d_t)> p_t(d_t)+\eps$ and output $\win$ else. However, note that this will violate privacy because $p_S(d_t)$ depends on the sample $S$. Therefore, the discriminator exploits a private mechanism that outputs $d_t$ only if $p_S(d_t)$ crosses a certain threshold: The mechanism we rely on has been studied in  
%\cite{dwork1}. Roughly, the proposed mechanism receives a data set of scalars $\Sigma =\{\sigma_i\}_{i=1}^m$ and the mechanism publishes $\top$ if $\sum_{i=1}^m \sigma_i > c+ O(1/N)$ where $N$ and $c$ are parameters of the mechanism and it publishes $\perp$ if $\sum_{i=1}^m \sigma_i < c- O(1/N)$: We then employ this mechanism over the sequence of scalars $\{d_t(x_1),\ldots, d_t(x_m)\}$

We next formally establish the privacy and utility guarantees of $D$.
In what follows, assume that the input sample $S$ satisfies \cref{{eq:size3}},
\paragraph{Privacy.}
The discriminator $D$ is a composition of two procedures, $M_1$ and~$M_2$, 
	where $M_1$ applies the PAP-PAC learner $M$ on the random subsample $S_\ell$, 
	and $M_2$ runs the procedure $\thrsh$.  
	Thus, the privacy guarantee will follow from the composition lemma (\cref{lem:composition}) 
	if we show that $M_1$ is $(6\tau \alpha(\tau m),4e^{6\tau\alpha(\tau m)}\tau\beta(\tau m))$-private  
	and $M_2$ is $(\tau,0)$-private. 
	The privacy guarantee of $M_1$ follows by applying\footnote{Note that in order to apply \cref{lem:amplification} on $M_1$, we need to assume that $M$ 
	satisfies $(\alpha,\beta)$ privacy with $\alpha\leq 1$. This assumption does not lose generality -- see the paragraph following the definition of Private PAC Learning.} 	
	\cref{lem:amplification} 
	with $v:=\lvert S\rvert$ and $n:= \lvert S_\ell\rvert = \tau\lvert S\rvert$,
	%is proven in Lemma 4.12 in \cite{Bun15thresholds}
	and the privacy guarantee of $M_2$ follows from the statement in \cref{fig:thrsh} since $\frac{N}{\lvert\Sigma\rvert} =\frac{|S_\ell|}{\lvert S\rvert} =\tau$.

\paragraph{Utility.}
%Assume that there exists $d\in \D$ such that $p_t(d)\le p_{S}(d)-\epsilon$, and that $S$ was drawn IID from some unknown distribution $\preal$. 
Let $q_t$ denote the distribution from which the subsample $S_\ell$ is drawn.
Note that by \cref{eq:size3}, $S_{\ell}= \tau\cdot \lvert S\rvert\geq m(\epsilon/8,\tau\delta/2)$. 
Therefore, since $M$ PAC learns $\D$, its output $d_t$ satisfies:
\[L_{q_t}(d_t) \le \min_{d\in \D\cup(1-\D)} L_{q_t}(d)+\frac{\eps}{8},\]
with probability at least~$1-\tau\delta/2$.
By \cref{eq:lossvsadvantage} this is equivalent to
\begin{equation}\label{eq:PAClearn}
p_S(d_t)- p_{t}(d_t) \geq \max_{d\in \D\cup(1-\D)}\bigl(p_{S}(d)-p_t(d)\bigr)-\epsilon/4.
\end{equation}
Now, by plugging in the statement in \cref{fig:thrsh}:
$(\Sigma,c,N):=(\{d_t(x)\}_{x\in S}, p_t(d_t)+\frac{5\epsilon}{8}, \lvert S_\ell\rvert)$, and~$\gamma:=\tau\delta/2$ 
and conditioning on the event that both $M$ and $\thrsh$ succeed (which occurs with probability at least $1-\tau\delta$)
it follows that
\begin{itemize}
     	\item[(i)] If $D$ outputs $d_t$ then 
	\[p_S(d_t)\geq   c - \frac{8\log(1/\gamma)}{N} =  p_t(d_t) + \frac{5\epsilon}{8} - \frac{8\log(\tau\delta/2)}{\tau\lvert S\rvert}\geq p_t(d_t) + \frac{\eps}{2},\]
	where in the last inequality we used that $\lvert S\rvert \geq \frac{64\log(\tau\delta/2)}{\eps\tau}$ (by \cref{eq:size3}).
        \item[(ii)] If $D$ outputs $\win$ then by a similar calculation $p_S(d_t) \leq p_t(d_t) + \frac{3\eps}{4}$ and therefore
        \[\mmd_\D(p_S,p_t) = \max_{d\in \D\cup(1-\D)}\bigl(p_S(d)-p_t(d)\bigr) \leq p_S(d_t) - p_t(d_t) + \frac{\eps}{4} \leq \eps,\]
	where in the first inequality we used \cref{eq:PAClearn}.
\end{itemize}
This concludes the proof of \cref{lem:discriminator}.
%
%%
%%Now, suppose there is $d\in\D\cup(1-\D)$ such that $p_{S}(d) - p_t(d)>\epsilon $.
%%Then, with probability at least~$(1-\tau\delta/2)$:
%%\[p_S(d_t) - p_{t}(d_t)\ge 3\epsilon/4.\]
%%
%%
%%Now the properties of $\thrsh$ as well as our choice of parameters ensure the following (w.p $(1-\tau\delta/2)$): 
%\begin{itemize}
%\item If $p_{S}(d_t)\le p_t(d_t)+\epsilon/2$, the algorithm outputs $\win$.
%\item If $p_{S}(d_t)\ge p_t(d_t)+3\epsilon/4$, the algorithm outputs $d_t$.
%\end{itemize}
%Thus, if the algorithm outputs $\win$ we have (w.p $(1-\tau\delta/2)$) that $p_S(d_t) < p_t(d_t)+3\epsilon/4$. In other words we have that $-3\epsilon/4 < p_t(d_t)-p_S(d_t)$.  The previous discussion then implies that (w.p. $(1-\tau\delta/2)$) $\sup_{d^*\in \D} p_S(d^*)-p_t(d^*)<\epsilon$. 
%
%Taken together we obtain that conditioned on the event that the algorithm outputs $\win$ w.p. $(1-\tau\delta)$, $\mmd_\D(p_t,p_S)<\epsilon$.
%
%Next, we condition on the event that the algorithm outputs $d_t$, then we have (w.p at least $(1-\tau\delta/2)$) that $p_t(d_t)- p_{S}(d_t)<-\epsilon/2$, as required. 

\begin{figure}
\fbox{\parbox{\textwidth}{
\begin{itemize}
    \item Let $M$ be a PAP-PAC learner for the class $\D\cup(1-\D)$ with sample complexity $m(\epsilon,\delta)$.
    \item Let $\epsilon,\delta,\tau$ be the input parameters.
    \item Let $S$ be the input sample, let $p_S$ be the uniform distribution over $S$, 
    and let $p_t$ be the distribution submitted by the generator.
%   \item If $|S|<1/\tau^{1.001}$ return $\top$ and exit. 
%    \item Set $m_0:=\tau\cdot m$
    \item Draw a labelled sample $S_{\ell}=\{({x}_i,y_i)\}$ of size $\tau\cdot \lvert S\rvert$ independently as follows: 
    draw the label~$y_i$ uniformly from $\{0,1\}$
    \begin{itemize}
	    \item[(i)] if $y_i=0$ then draw ${x}_i\sim p_{S}$, 
	    \item[(ii)] if $y_i=1$ then draw ${x}_i \sim p_t$.
    \end{itemize}
    \item Apply the learner $M$ on the sample $S_{\ell}$ and set $d_t\in \D$ as its output.
    \item Compute $Z:=\thrsh\left(\{d_t(x)\}_{x\in S},p_t(d_t)+\frac{5\epsilon}{8},\lvert S_{\ell}\rvert\right)$.
    \begin{itemize}
    \item[(i)] If $Z= \top$ then send the generator with $d_t$, 
    \item[(ii)] else,  $Z=\perp$  and reply the generator with ``Win''.
    \end{itemize}
%    \item If $\thrsh\left(\{d_t(x)\}_{x\in S},p_t(d_t)+\frac{5\epsilon}{8},N\right)= \top$ then send the generator with $d_t$; 
%    else reply "Win".
\end{itemize}}}
\caption{Depiction of the private discriminator used in \cref{thm:main3}.
The discriminator holds the target distribution $p_S$, where $S$ is a sufficiently large sample from $\preal$.
In each round the discriminator decides whether $p_S$ is indistinguishable from the distribution submitted
by the generator and replies accordingly.}\label{alg:discriminator}

\fbox{\parbox{\textwidth}{
\paragraph{$\thrsh$.} The procedure $\thrsh$ receives as input a dataset of scalars $\Sigma=\{\sigma_i\}$, a threshold parameter $c>0$ and a margin parameter $N$ and has the following properties (see Theorem 3.23 in \cite{dwork1} for proof of existence): 

\begin{itemize}
    \item $\thrsh(\Sigma,c,N)$ is $(N/|\Sigma|,0)$-private.
    \item For every  $\gamma>0$:
    \begin{itemize}
    	\item If $\frac{1}{|\Sigma|} \sum_{\sigma \in \Sigma} \sigma > c+ \frac{8\log 1/\gamma}{N}$ then $\thrsh$ outputs $\top$ with probability at least $1-\gamma$
     	\item If $\frac{1}{|\Sigma|} \sum_{\sigma \in \Sigma} \sigma < c - \frac{8\log 1/\gamma}{N}$ then $\thrsh$ outputs $\perp$ with probability at least $1-\gamma$
	\end{itemize}
\end{itemize}}}
\caption{The procedure: $\thrsh$}\label{fig:thrsh}
\end{figure}

\end{proof}

\paragraph{\ref{it:foolable}$\Rightarrow$\ref{it:sanitizable}.}
We now prove that $DP$-Fooling entails sanitization.
\begin{proposition}\label{lem:quant:2}
Let $\D$ be a class that is $DP$-Foolable via a fooling algorithm with privacy parameters $(\alpha(m),\beta(m))$ then $\D$ is sanitizable with privacy parameters $\tilde{\alpha}(m),\tilde{\beta}(m)=(12\alpha(m),e^{12\alpha(m)}8\beta(m))$ private and sample complexity $2m(\epsilon,\delta)$
\end{proposition}
\begin{proof}
Let $A$ be a \creative--Fooling algorithm. Consider an algorithm $B$ that, given a sample $S$ of size $2m(\epsilon,\delta)$, subsamples $m(\epsilon,\delta)$ points (with replacement) and runs $A$ on the subsample. Since the sample is drawn i.i.d from the distribution $p_S$ we obtain that, by the guarantees of $A$, with probability $(1-\delta)$:
\[\mmd_{\D}(p_{syn},p_S)< \epsilon.\] In particular, the function $\textrm{EST}(d)=\mathbb{E}_{x\sim p_{syn}}[d(x)]$ sanitizes $S$.

By \cref{lem:amplification} we obtain that the algorithm $B$ is $(12\alpha(m),e^{12\alpha}8\beta(m))$ private.
\end{proof}

\paragraph{\ref{it:sanitizable}$\Rightarrow$\ref{it:puc}}
We next prove that sanitization entails private uniform convergence.
\begin{proposition}\label{lem:quant:3}
Let $\D$ be a class that is sanitizable with a sanitizer that has sample complexity $m(\epsilon,\delta)$ and privacy parameters $(\alpha(m),\beta(m))$ then $\D$ has finite VC dimension and has the private uniform convergence. Specifically there exists a differentially private algorithm $M$, with privacy parameters $(\tilde{\alpha}(m),\tilde{\beta}(m))= (2\alpha(m/16)+1/m,\beta(m/16))$ and sample complexity
\[\hat m(\epsilon,\delta)= O\left(m(\epsilon/18,\delta/6) + \frac{\textrm{VC}(\D)\log 1/\delta}{\epsilon^2}\right).\] that outputs $\hat{L}:\D\to [0,1]$ such that with probability $(1-\delta)$
\[ (\forall d\in \D)~:~ |\hat{L}(d)-L_{\P}(d)|<\epsilon .\]
\end{proposition}
\begin{proof}
The proof is very similar to Lemma 5.4 in \cite{beimel2013private} (which addresses only the pure case). First, we want to show that the VC dimension is bounded. We can apply Theorem 5.12 in \cite{vadhan} which asserts that for any domain $\X$ we can find $k$ counting queries for which any differentially private mechanism with parameters $\alpha= 1$ and $\beta=0.1$ that estimates all $k$ counting queries, within error at most $\epsilon$, must observe a sample size that scales with $k$. A sanitizer $M$ over a class $\D$ with VC dimension $d$ returns all $2^d$ counting queries over a domain $\X_d$ of size $d$, we can use the above to obtain a bound over the maximal size of a shattered set, and we obtain that any sanitizable class must have finite VC dimension.

For the rest of the proof we will need the following notations. First, given a sample $S=\{x_i,y_i\}_{i=1}^m$, let $S^{-}=\{x_i,1-y_i\}$ and $\hat{S}$ be a sample with $S$, concatenated with $S^{-}$. 

Next, we will denote by $u(\epsilon,\delta)=\theta(\frac{\textrm{VC}(\D)\log 1/\delta}{\epsilon^2})$ the sample complexity for standard uniform convergence for the class $\D$: namely we assume that for any unknown distribution, $\P$, given $u(\epsilon,\delta)$ examples drawn i.i.d we have that with probability at least $(1-\delta)$:

\begin{equation}\label{eq:uc} 
\forall d\in \D: |L_\P(d) -L_S(d)|<\epsilon.
\end{equation}

We can also assume that $u(\epsilon,\delta)\ge 8/\epsilon\log 1/\delta$. Next, we denote by $M_{count}$ a (2/m,0)-private mechanism that given a labelled sample $S$ of size $m\ge 8/\epsilon\log 1/\delta\ge u(\epsilon,\delta)$ returns w.p $(1-\delta)$ a number $p$ such that
\[ \left|p-\frac{|\{x_i: (x_i,1)\in S\}|}{|S|}\right|\le \epsilon.\]
For concreteness one may apply the Laplace mechanism (see for example \cite{vadhan,dwork3} for further details). In particular, if $S$ is drawn i.i.d from some unknown distribution $\P$ and $m=u(\epsilon/2,\delta/2)=\Omega(8/\epsilon\log 1/\delta)$ then by union bound we have that w.p $1(-\delta)$:
\[ |p-\P(y=1)|\le \epsilon.\]
We next depict the algorithm for $\hat{L}$:

\fbox{\parbox{\textwidth}{
\begin{enumerate}
\item \textbf{Assume:} \\ A sample $S$ of size $16\cdot \max\{ m(\epsilon/12,\delta/6,u(\epsilon/18,\delta/6)\}$, drawn i.i.d from some unknown distribution $\P$. \\A sanitizer $M$ for the class $\D$ with privacy parameters $(\alpha(m),\beta(m))$ and sample complexity $m(\epsilon,\delta)$
\item Let $\hat{S}$ be the first $m_1=\max\{m(\epsilon/12,\delta/6),u(\epsilon/12,\delta/6)\}$ elements in the sample $S$.

\item Provide $M$ with the sample $\hat{S}$ and obtain, w.p. $1-\delta/3$ a function $\textrm{EST}'$ such that
\begin{align}\label{eq:EST}\forall d\in \D,&
 \left|\textrm{EST}'(d) - \frac{|\{x\in \hat{S}: d(x)=0\}|}{m_1}\right|<\epsilon/6,\\
 &\left|L_\P(d) -\frac{|\{x\in \hat{S}: d(x)=0\}|}{m_1}\right|<\epsilon/6 \label{eq:EST2}
\end{align}
\item Apply $M_{count}$ on $S$ to obtain, w.p. $1-\delta/3$:
\begin{align}\label{eq:p11}
&\left|p- \frac{|(x_i,y_i)\in S: y_i=1\}|}{|S|}\right|<\epsilon/18 \\
&\left|\frac{|\{(x_i,y_i)\in S: y_i=1\}|}{|S|}-\P(y=1)\right|<\epsilon/18
\label{eq:p12}
\end{align}.
\item If $p \ge 1/8$, let $\sigma=1$ and set $p_\sigma=p$, else $\sigma=0$ and set $p_\sigma=1-p$.
\item If $\sigma=0$ set $\textrm{EST}=\textrm{EST}'$ and if $\sigma=1$ set $\textrm{EST}=1-\textrm{EST}'$.
\item Let $\hat{S}_\sigma$ be the first $m_2=\max\{m(\epsilon/12,\delta/6),u(\epsilon/12,\delta/6)\}$ elements of $\hat{S}$ with label $\sigma$
\item Provide $M$ with $\hat{S}_\sigma$ and obtain $\textrm{EST}_\sigma$ such that w.p. $(1-\delta/3)$:
\begin{align}\label{eq:EST_SIGMA}
\forall d\in \D, &\left|\textrm{EST}_\sigma(d) - \frac{|\{x\in \hat{S}_\sigma: d(x)=\sigma\}|}{m}\right|<\epsilon/12.\\
&\left|\P(d(x)=\sigma|y=\sigma) - \frac{|\{x\in \hat{S}_\sigma: d(x)=\sigma\}|}{m}\right|<\epsilon/12 \label{eq:EST_SIGMA2}
\end{align}
\item Set
\[\hat{L}(d)= \textrm{EST}(d)+p_\sigma -2p_\sigma\cdot \textrm{EST}_\sigma(d).\]
\end{enumerate}
}}

\textbf{Utility}
First, note that by the choice of sample size, the fact that $S$ is i.i.d and by union bound we have that with probability at least $(1-\delta)$ \cref{eq:EST,eq:EST2,eq:p11,eq:p12,eq:EST_SIGMA,eq:EST_SIGMA2} all hold. Indeed \cref{eq:EST,eq:EST_SIGMA} each hold with probability $1-\delta/6$, since $m_1,m_2\ge m(\epsilon/18,\delta/6)$. Also, \cref{eq:EST2,eq:p11,eq:p12} each hold, w.p. $(1-\delta/6)$ because $m_1,m_2\ge u(\epsilon/18,\delta,6)$. For \cref{eq:EST_SIGMA2} we claim that , contingent on \cref{eq:p11,eq:p12}, $\hat S_{\sigma}$) is indeed an i.i.d sample from the distribution $\P(\cdot|y=\sigma)$. Indeed, there are more than $m_2$ $\sigma$-labelled elements in $S$, hence by the definition of $\hat{S}_\sigma$ the input sample is indeed an i.i.d sample of size $m_2$ distributed according to the conditional distribution.

By traingular inequality we then obtain:
\begin{enumerate}
\item  For all $d\in \D$: $\left|\textrm{EST}'(d)-L_\P(d)\right|<\epsilon/3$
\item $\left|p_\sigma-\P(y=\sigma)\right|<\epsilon/9$
\item For all $d\in \D$: $\left|\textrm{EST}_\sigma(d) -\P(d(x)=\sigma|y=\sigma)\right|\le \epsilon/6$.
\end{enumerate}
We will thus assume that these events happened. 
Next note that we have:
\begin{align*} \P(d(x)\ne y)&= \P(y=1-\sigma\wedge d(x)=\sigma)+\P(y=\sigma\wedge d(x)=1-\sigma)\\
&= \P(d(x)=\sigma)- \P(y=\sigma \wedge d(x)=\sigma)+\P(y=\sigma)\cdot (1-\P(d(x)=\sigma|y=\sigma))\\
&= \P(d(x)=\sigma)- \P(y=\sigma)\cdot \P(d(x)=\sigma|y=\sigma)
+\P(y=\sigma)\cdot (1-\P(d(x)=\sigma|y=\sigma)).\\
&= \P(d(x)=\sigma)+ \P(y=\sigma)- 2\P(y=\sigma)\cdot \P(d(x)=\sigma|y=\sigma)
\end{align*}
We thus have

\begin{align*}|\hat{L}(d)- \P(d(x)\ne y)|&=
|\P(d(x)=\sigma)-\textrm{EST}(d)+ \P(y=\sigma)-\P_\sigma - 2\left(\P(y=\sigma)\cdot \P(d(x)=\sigma|y=\sigma)- \P_\sigma \textrm{EST}_\sigma(d)\right)\\
&\le  
|\P(d(x)=\sigma)-\textrm{EST}(d)|+ |\P(y=\sigma)-p_\sigma| +2|\P(y=\sigma)\cdot \P(d(x)=\sigma|y=\sigma)- p_\sigma \textrm{EST}_\sigma(d)|\\
&\le 
|\P(d(x)=\sigma)-\textrm{EST}(d)|+ 3|\P(y=\sigma)-p_\sigma| +2 |\P(d(x)=\sigma|y=\sigma)-\textrm{EST}_\sigma(d)|\\
&\le \epsilon/3 + \epsilon/3+\epsilon/3.
\end{align*} 

\textbf{Privacy}
Note that if $S$ and $S'$ are two sample sets of Hamming distance one (i.e. differ by a single example) then $\hat{S}_\sigma$ and $\hat{S}'_{\sigma}$ are also of distance one. Hence applying the mechanism $M$ on $\hat{S}_{\sigma}$ preserve $(\alpha,\beta)$ differentially privacy as a function over the sample $S$.
Taken together, by post processing and composition, we obtain that overall as we apply three different private mechanisms and obtain privacy guarantee of $(2\alpha(m/16)+1/m,2\beta(m/16))$.
\end{proof}
\ignore{\paragraph{\ref{it:sanitizable}$\Rightarrow$\ref{it:pac}.}
This follows from Theorem 5.5 in \cite{beimel2013private}. In particular we have the following corollary

\begin{proposition}\label{lem:quant:4}
Let $\D$ be a class that is sanitizable with sample complexity $m(\epsilon,\delta)$ then $\D$ is privately and properly learnable with sample complexity $O\left(m(\epsilon,\delta)+ \frac{\log 1/\delta}{\epsilon}\right)$.
\end{proposition}}

\paragraph{\ref{it:puc}$\Rightarrow$\ref{it:pac}.}
The final entailment, that private uniform convergence implies learnability is immediate and is an immediate corollary of post-processing for differential privacy (\cref{lem:pp}).
	Indeed, by the private uniform convergence property we can privately estimate the losses
	of all hypotheses in $\D$, and then output any hypothesis in $\D$ that minimizes the estimated loss.
%Given a sample $\{x_i\}_{i=1}^m$ simply consider a labelled sample $\{(x,1)\}_{i=1}^m$ and set $\mathrm{EST}(h)=L_S(h)$.
To conclude have the following
\begin{proposition}\label{lem:quant:4}
Let $\D$ be a class with the private uniform convergence property with parameters $(\alpha(m),\beta(m)$ and sample complexity $m(\epsilon,\delta)$, then $\D$ is PAP-PAC learnable with privacy parameters $(\alpha(m),\beta(m))$ and sample complexity $m(\epsilon,\delta)$.
\end{proposition}
\ignore{\paragraph{\ref{it:pac}$\Rightarrow$\ref{it:puc}.}
\begin{proposition}\label{lem:quant:1}
Let $\D$ be a class that is PAP-PAC learnable with a learner $L$ that has sample complexity $m(\epsilon,\delta)$ with privacy parameters $(\alpha(m),\beta(m))$. Then, for every $0<\kappa<1$ there exists a differentially private algorithm $M$, with privacy parameters $(\alpha(m^{1-\kappa}),e^{O(\alpha^{1-\kappa})}O(\beta^{1-\kappa}))$ and sample complexity
\[ m_\kappa(\epsilon,\delta)=\tilde O\left(m\left(\epsilon/8,\frac{\ell^* \delta/2}{\epsilon^2}\log\frac{\ell^*}{\epsilon^2}\right)\cdot\frac{\ell^*}{\epsilon^2}+\left(\frac{\ell^*}{\epsilon^2}\right)^{\kappa}\right),\]
where $\ell^*$ is the dual Littlestone dimension of $\D$,
that outputs $\hat{L}:\D\to [0,1]$ such that 
\[ (\forall d\in \D)~:~ |\hat{L}(d)-L_{\P}(d)|<\epsilon .\]
In particular, $\D$ has the private uniform convergence property.
\end{proposition}

Suppose $\D$ is PAP-PAC learnable by an algorithm $L$. 
	For every function $d\in \D$, let $d'$ denote the $(X\times\{0,1\})\to\{0,1\}$ function defined by $d'((x,y))= \mathbf{1}[d(x){\neq}y]$,
	and let $\D' = \{ d' : d\in\D\}$.
	Observe that for every sample $S\subseteq (X\times\{0,1\})^m$:
	\begin{equation}\label{eq:graph}
	L_S(d) = p_S(d'),
	\end{equation} 
	 where $L_S(d)$ denotes the empirical loss of $d$ and~$p_S$ denotes the empirical measure of $d'$. 
	
We claim that $\D'$ is also PAP-PAC learnable with sample complexity $m(\epsilon,\delta)$: 
	for a $\D'$-example $z'=((x,y),y')$ let $z$ denote the $\D$-example $(x,\lvert y'-y\rvert)$,
	and note that $d'$ errs on $z'$ if and only if $d$ errs on $z$.
	Therefore, a PAP-PAC learner for $\D'$ follows by using this transformation to convert the $\D'$-input sample $S' = \{z'_i\}_{i=1}^m$
	to a $\D$ input sample $S=\{z_i\}_{i=1}^m$, applying $L$ on $S$ and outputting $d'$, where $d=L(S)$.
	
%	apply $A$ on the sample $S=\{(x_i, \lvert y_i'-y_i\rvert \}_{i=1}^m$ and output $d'(x,y)=\mathbf{1}[d(x)\neq y]$, where $d=A(S)$.  

	Taken together we have that, by \ref{it:pac}$\implies$\ref{it:sanitizable} it follows that $\D'$ is sanitizable by a sanitizer $M$ with sample complexity~$m_\kappa(\eps,\delta)$.
	We next use $M$ to show that $\D$ satisfies private uniform convergence:
	let $\P$ be a distribution over $\X\times\{0,1\}$ and $\eps,\delta$ be the error and confidence parameters.
	Consider the following algorithm:
	\begin{itemize} 
	\item Draw a sample $S$ from $\P$ of size $m_\kappa(\eps/2,\delta/2)$
	\item Apply $M$ on $S$ to obtain an estimator $\mathrm{EST}':\D'\to[0,1]$ and output the estimator $\mathrm{EST}:\D\to[0,1]$
	defined by $\mathrm{EST}(d) = \mathrm{EST'}(d')$.
	\end{itemize}
	We want to show that 
	\[(\forall d\in \D):\lvert \mathrm{EST}(d) - L_\P(d) \rvert \leq {\eps},\]
	with probability $1-\delta$.
	Indeed, since $m_\kappa\geq \frac{\mathrm{VC}(\D)+\log(2/\delta)}{\frac{\eps^2}{4}}$ it follows that 
	\[(\forall d\in \D):\lvert L_S(d) - L_\P(d) \rvert \leq \frac{\eps}{2},\]
	with probability at least $1-\frac{\delta}{2}$, 
	and since $m\geq m_1(\frac{\eps}{2},\frac{\delta}{2})$,  
	\begin{align*}
	(\forall d\in \D):	\lvert \mathrm{EST}(d) - L_S(d) \rvert &= \lvert \mathrm{EST}'(d') - p_S(d') \rvert \tag{by \cref{eq:graph}}\\
											 &\leq \eps/2,
	\end{align*}
	with probability $1-\frac{\delta}{2}$.
	The desired bound thus follows by a union bound and the triangle inequality.

%Now, given a sample $S=\{(x_i,y_i)\}_{i=1}^m$ by sanitzability of $\D'$ we can privately output an estimator $\mathrm{EST}:\D\to[0,1]$ such that
%	$\lvert \mathrm{EST}(d) - L_S(d) \rvert \leq \eps/2$ for all $d\in \D$, with probability $1-\frac{\delta}{2}$.
%%\[ \mathrm{EST}(h) \le \frac{|(x,y)\in S; h(x)=y|}{|S|} + \epsilon\]
%	Next, because $\D$ has a finite VC dimension (by PAC learnability), 
%	given any distribution $\P$ over $X\times \{0,1\}$ and a sample $S\sim \P^m$ for $m=O(\frac{\mathrm{VC}(\D)+\log(1/\delta)}{\eps^2})$
%	
%	if $S$ is drawn IID from some unknown distribution $\P$ (and $m= O(\ell/\epsilon^2)\log 1/\delta$) we obtain w.p $(1-\delta)$ 
%\[\left|\EE_{(x,y)\in \P} [\mathbf{1}[h(x)= y]] -\frac{|(x,y)\in S; h(x)=y|}{|S|}\right| \le \epsilon.\]
%Thus, setting $L_S(h)=\mathrm{EST}(h)$ we obtain the desired result.
} %ADD FOR FULL VERSION
\bibliographystyle{abbrvnat}
\bibliography{gams}
\appendix

\section{Proof of \cref{cor:bendavid}}\label{sec:bendavid}

We begin by defining the predictors $\hat f_t$'s that $L$ uses: let $L_0$ be the learner implied by \cref{thm:bendavid}. 	
	We first turn $L_0$ into a deterministic learner whose input is $(p_1, y_1),\ldots, (p_T, y_T)\in \Delta(\W)\times\{0,1\}$ 
	and that outputs at each iteration $f_t:\W\to [0,1]$. Then, we extend $f_t$ linearly to $\hat f_t$ as discussed in \cref{sec:basics}. 
	Let $(p_1, y_1),\ldots, (p_T, y_T)\in \Delta(\W)\times\{0,1\}$, given $w\in \W$, the value $f_t(w)$ is the expected output of the following 	random process:
\begin{itemize}
\item sample $w_i\sim p_i$ for $i\leq t-1$,
\item apply $L_0$ on the sequence $(w_1,y_1),\ldots,(w_{t-1}, y_{t-1})$ to obtain the predictor $\tilde f_t$, and
\item output $\tilde f_t(x)$.
\end{itemize}
	That is,
	\[f_t(x) = \EE_{w_{1:t-1}} \Bigl[\EE_{\tilde f_t\sim L_0} [\tilde f_t(w)~\Big\vert~ x_1\ldots x_{t-1} \Bigr] ],\]
	where $\EE_{p_{1:t}}[\cdot]$ denotes the expectation over sampling each $w_i$ from $p_i$ independently, 
	and $\EE_{\tilde f_t\sim L_0}[\cdot]$ 	denotes the expectation over the internal randomness of the algorithm $L_0$ at iteration $t$.
	Finally,~$\hat f_t(p) = \EE_{w\sim p}[f_t(w)]$ is the predictor that $L$ uses at the $t$'th round.
	Note that indeed $\hat f_t$ is determined (deterministically) from $(p_1,y_1),\ldots (p_{t-1},y_{t-1})$.

We next bound the regret: for every $h\in \H$:
\begin{align*}
\sum_{t=1}^T|\hat f_t(p_t)-y_t| -|\hat h(p_t)- y_t| =&\sum_{t:y_t=0} \hat f_t(p_t) - \hat h(p_t)+\sum_{t:y_t=1} \hat h(p_t)-\hat f_t(p_t) \\
							    =& \sum_{\{t:y_t=0\}} \EE_{p_{1:t-1}} \Bigl[\EE_{L_0} [\EE_{p_t} [f_t(w_t)] ~\Big\vert~ \{w_i\}_{i=1}^{t-1} ]\Bigr] - \EE_{p_{1:T}}\left[h(x_t)\right]\\
								&+\sum_{\{t: y_t=1\}} \EE_{p_{1:T}}\left[h(w_t)\right]-\EE_{p_{1:t-1}}\Bigl[\EE_{L_0} [\EE_{p_t}[f_t(w_t)]~\Big\vert~ \{x_i\}_{i=1}^{t-1}]  \Bigr]\\							    
							    =& \sum_{\{t:y_t=0\}} \EE_{p_{1:T}} \Bigl[\EE_{L_0} [f_t(x_t)~\Big\vert~ \{w_i\}_{i=1}^T]  \Bigr] - \EE_{p_{1:T}}\left[h(w_t)\right]\\
							    &+\sum_{\{t: y_t=1\}} \EE_{p_{1:T}}\left[h(w_t)\right]-\EE_{p_{1:T}}\Bigl[\EE_{L_0} [f_t(w_t)~\Big\vert~ \{w_i\}_{i=1}^T]  \Bigr]\\
							   %&=\EE_{p_{1:T}}\left[\sum_{y_t=0} \EE_{L_0}\left[f_t(w_t)\right] - h(w_t)+\sum_{y_t=1} h(w_t)-\EE_{L_0} \left[f_t(w_t)\right] \right]\\
							   =&\EE_{p_{1:T}}\Biggl[\EE_{L_0}\Bigl[\sum_{y_t=0} f_t(w_t) - h(x_t)+\sum_{y_t=1} h(w_t)- f_t(w_t)~\Big\vert~ \{w_i\}_{i=1}^T \Bigr] \Biggr]\\
							   =& \EE_{p_{1:T}}\Biggl[\EE_{L_0}\Bigl[\sum_{t=1}^T\lvert f_t(w_t)-y_t\rvert -\lvert h(w_t)-y_t\rvert\Big\vert~ \{w_i\}_{i=1}^T\Bigr] ~\Biggr]\\
							   \le& \EE_{p_{1:T}} \bigl[\mathrm{REGRET}_T(L_0,\{w_t,y_t\}_{t=1}^T \bigr]\\
							   \le& \sqrt{\frac{1}{2}\ell T \log T}.
\end{align*}

%
%\begin{align*}
%\sum_{t=1}^T|\hat f_t(p_t)-y_t| -|\hat h(p_t)- y_t| &=\sum_{t:y_t=0} \hat f_t(p_t) - \hat h(p_t)+\sum_{t:y_t=1} \hat h(p_t)-\hat f_t(p_t) \\
%							    &= \sum_{\{t:y_t=0\}} \EE_{p_{1:T}} \left[\EE_{L_0}\left[f_t(x_t)\right]\right] - \EE_{p_{1:T}}\left[h(x_t)\right]\\
%							    &+\sum_{\{t: y_t=1\}} \EE_{p_{1:T}}\left[h(x_t)\right]-\EE_{p_{1:T}}\left[\EE_{L_0} \left[f_t(x_t)\right] \right]\\
%							   %&=\EE_{p_{1:T}}\left[\sum_{y_t=0} \EE_{L_0}\left[f_t(x_t)\right] - h(x_t)+\sum_{y_t=1} h(x_t)-\EE_{L_0} \left[f_t(x_t)\right] \right]\\
%							   &=\EE_{p_{1:T}}\EE_{L_0}\left[\sum_{y_t=0} f_t(x_t) - h(x_t)+\sum_{y_t=1} h(x_t)- f_t(x_t) \right]\\
%							   &= \EE_{p_{1:T}}\EE_{L_0}\left[\sum_{t=1}^T\lvert f_t(x_t)-y_t\rvert -\lvert h(x_t)-y_t\rvert\right]\\
%							   &\le \EE_{p_{1:T}} \left[\mathrm{REGRET}_T(L_0,\{x_t,y_t\}_{t=1}^T\right]\\
%							   &= O(\sqrt{\ell T \log T}).
%\end{align*}
%\qed

\section{Extending \Cref{thm:main1quant}, \cref{it:ub} to infinite classes}\label{sec:infinite}

Here we extend the proof of the upper bound in \cref{thm:main1quant} to the general case where either $\X$ or~$\D$ may be infinite.
The proof follows roughly the same lines like the finite case.
The first technical milestone we need to consider is to properly define a $\sigma$-algebra over the domain $\D$ and specify the space $\Delta(D)$ of probability measures. 
For this, we consider $\{0,1\}^\X$ as a topological space with an appropriately defined topology and $\Delta(D)$ as the space of Borel-probability measures. 
We refer the reader to \cref{apx:preliminaries} for the exact details. %but we now continue outlining the proof with theses definitions at hand.

We will also make some technical modifications in the protocol depicted in \cref{alg:main}. The modification is depicted in \cref{fig:modification}.
\begin{figure}[h]
\fbox{\parbox{\textwidth}{
Consider \cref{alg:main} with the following modification, at the \textbf{Else} Step: 
\begin{itemize} 
	\item Find $\bar{d}_t \in \Delta(\D)$, \textbf{with finite support} such that 
\[\bigl(\forall x\in \X \bigr): \EE_{d \sim \bar{d}_t}\left[ f_t(d) - x(d)\right] > \frac{\epsilon}{4}\] 
(if no such $\bar{d}_t$ exists then output {\it ``error''}).
\end{itemize}
}}
\caption{Modifying \cref{alg:main}}\label{fig:modification}
\end{figure}
The first modification we make is that in the $\textbf{Else}$ step, the generator chooses $\bar{d}_t$ with finite support. For the finite case, the requirement that $\bar{d}_t$ has finite support is met automatically. The second modification we make allows further slack in the distinguisher. Instead of requiring $> \frac{\epsilon}{2}$ we allow $>\frac{\epsilon}{4}$. Clearly this change in constant does not change the asymptotic regret bound.

\paragraph{Proof outline.}
To extend the proof to the infinite case it suffices to ensure 
	that the generator in \cref{alg:main} (with the modification in \cref{fig:modification}) never outputs \emph{``error''} in the 2nd item of the ``For'' loop. 
	To be precise, let us add the following notation that is consistent with the algorithm in \cref{alg:main}. 
	%Given a $\sigma$-algebra over the domain $\D$ let us consider the space of probability  denoted by $\Delta(D)$
	Let $f:\D\to[0,1]$ be measurable.

\begin{enumerate}
    \item If there exists $p\in \Delta(\X)$ such that
\[ (\forall d\in \D):\EE_{x\sim p}[f(d)-x(d)] \le \frac{\epsilon}{2}, \] 
     we say that $f$ satisfies \cref{it:if}. 
    \item If there exists $\bar d\in \Delta(\D)$ such that 
\[\bigl(\forall x\in \X \bigr): \EE_{d \sim \bar{d}}\left[ f(d) - x(d)\right] > \frac{\epsilon}{2}\] 
    we say that $f$ satisfies \cref{it:else}.
    \item $f$ is \emph{amenable} if it satisfies either \cref{it:if} or \cref{it:else}.
\end{enumerate}
%Note that indeed if the generator in \cref{alg:main} uses an online learner $\A$ only uses amenable $f_t$'s then it never outputs ``error''. 
	
%Note that every $f$ satisfies at most one of items 1 and 2 (this essentially follows from weak LP duality): indeed if \cref{it:if} is satisfied and witnessed by $p$, then for every $\bar d \in\Delta(\D)$
%\begin{align*} f(\bar d) - \EE_{x\sim p}[x(\bar d)] &= f(\bar d)  - \EE_{x\sim p, d\sim \bar d}[x(d)]  
%						    \\&=  f(\bar d)   - \EE_{d\sim \bar d}[p(d)]
%						    \\& = \EE_{d\sim \bar d}[f(d)-p(d)] \leq \EE_{d\sim \bar d} [{\eps}/{2}]\ = {\eps}/{2},
%\end{align*}
%and so there can be no $\bar d$ as required by \cref{it:else}. 
%A similar argument yields the converse.
%Conversely, if \cref{it:else} is satisfied and witnessed by~$\bar d$
%then for every $p\in \Delta(\X)$,
%\[f(\bar d) - p(\bar d) = \EE_{d\sim \bar d} [f(d) - p(d)] >{\eps}/{2},\]
%and in particular there must be some $d\in \D$ such that $f(d)-p(d)>\frac{\eps}{2}$,
%so there can be no $p$ as required by \cref{it:if}.
When $\X$ and $\D$ are finite, every $f$ satisfies one of Items 1 or 2 (and hence amenable).
	This is the content of \cref{lem:vonneuman} which is proved using strong duality (in the form of the Minmax Theorem).
	However, the case when $\X$ and $\D$ are infinite is more subtle. 
	Specifically, the Minmax Theorem does not necessarily hold in this generality.
	
	The next lemma guarantees the existence of a learner $\A$ 
	which only outputs amenable functions. 
%	Then, the proof is exactly the same as in the finite case.
	Recall that~$\hat f:\Delta(\D)\to[0,1]$ denotes the linear extension of $f$  
	and is defined by $\hat f(\bar d) = \EE_{d \sim \bar d}[f(d)]$.

\begin{lemma}\label{lem:minmax_infinite}
Let $\D$ be a discriminating class with dual Littlestone dimension $\ell^*$, and let $T$ be the horizon. 
	Then, there exists a deterministic online learning algorithm $\A$ for the dual class~$\X$ 
	that receives labelled examples from the domain $\Delta(\D)$ and uses predictors of the form $\hat f_t$ for some~$f_t:\D\to[0,1]$,
	such that:
\begin{enumerate}
        \item $A$'s regret is $O(\sqrt{\ell^*T\log T})$, and
        \item For all $t\le T$, if the sequence of observed examples $(\bar{d}_1,y_1),\ldots, (\bar{d}_{t-1},y_{t-1})$ up to iteration $t$, all have finite support then $A$ chooses $f_t$ that is amenable (in particular $f_1$ is also amenable).
\end{enumerate}
\end{lemma}

Our next Lemma shows that \cref{alg:main} with the modification depicted in \cref{fig:modification} will indeed never output error:
\begin{lemma}\label{lem:modification}
Consider \cref{alg:main} with the modification depicted in \cref{fig:modification}. Assume $\A$ satisfies the properties in \cref{lem:minmax_infinite}. The for all $t\le T$ the generator never outputs error.
\end{lemma}
\begin{proof} The proof follows by induction, for $t=1$ the amenability of $f_1$ ensures that if $f_1$ doesn't satisfy \cref{it:if} then there exists $\bar{d}\in \Delta(\D)$ that satisfy \cref{it:else}. Now recall that $\X$ has finite Littlestone dimension and in particular finite VC dimension, by uniform convergence it follow that there is a finite sample $d_1,\ldots, d_m$ such that

\[\sup_{x\in \X}\left| \EE_{d\sim \bar{d}} \left[f_1(d)-x(d)\right] - \frac{1}{m} \sum_{i=1}^m f_1(d_i)-x(d_i)\right| \le \frac{\epsilon}{4}\] 
We then choose $\bar{d}_1$ to be a uniform distribution over $d_1,\ldots, d_m$. By the condition in \cref{it:else} and the above equation we obtain that 
\[\EE_{d \sim \bar{d}_1}\left[ f(d) - x(d)\right] >\frac{\epsilon}{4}\]

We continue with the induction step, and consider $t=t_0$. Note that by construction at each iteration up to iteration $t_0$ the algorithm $\A$ observed only distributions with finite support. In particular, we have that $f_{t_0}$ will be amenable. Hence, if it doesn't satisfy \cref{it:if} then we again obtain $\bar{d}$ that satisfies \cref{it:else}. We next discretize $\bar{d}$ as before. Using the finite VC dimension of $\X$ we obtain $\bar{d}_{t_0}$ that has finite support and satisfies:
\[\EE_{d \sim \bar{d}_{t_0}}\left[ f(d) - x(d)\right] >\frac{\epsilon}{4}\]
\end{proof}

\cref{lem:minmax_infinite}, together with \cref{lem:modification}, implies the upper bound in \cref{thm:main1quant}, \cref{it:ub} via the same argument as in the finite case.
This follows by picking the online learner used by the generator in \cref{alg:main} as in \cref{lem:minmax_infinite};
the amenability of the $f_t$'s (and \cref{lem:modification}) implies that the protocol never outputs ``error'', and the rest of the argument is exactly the same like in the finite case (with slight deterioration in the constants).
\begin{corollary}
Let $A$ be an algorithm like in the above Lemma. Then, if one uses $A$ as the online learner in the algorithm in \cref{alg:main}, together with the modification in \cref{fig:modification}, then the round complexity of it is at most 
$O(\frac{\ell^*}{\eps^2}\log\frac{\ell^*}{\eps})$, as in \cref{thm:main1quant}, \cref{it:ub}.
\end{corollary}

In the remainder of this section we prove \cref{lem:minmax_infinite}.

\subsection{Preliminaries}\label{apx:preliminaries}

We first present standard notions and facts from topology and functional analysis that will be used. We refer the reader to \cite{rudin1, rudin2} for further reading.

\paragraph{Weak* topology.} Given a compact Haussdorf space $K$, let $\Delta(K)$ denote the space of Borel measures over $K$, and let $C(K)$ denote the space of continuous real functions over $K$.
The weak* topology over $\Delta(K)$ is defined as the weakest\footnote{In the sense that every other topology with this property contains all open sets in the weak* topology.} topology so that for any continuous function~$f\in C(K)$ the following ``$\Delta(K)\to\mathbb{R}$'' mapping is continuous
\[T_{f}(\mu) = \int f(k) d\mu(k) .\]

%An important property\footnote{For separable spaces this property is equivalent to the definition of the weak* topology, but we will avoid making this assumption and work with the original definition.} of the weak* topology is that whenever a sequence of measures~$\{\mu_n\}_{n=1}^\infty$ converges (in the sense of the weak* topology) to a measure $\mu$ then $\mu_n(f)\to \mu(f)$ for every~$f\in C(K)$.

We will rely on the following fact, which is a corollary of Banach--Alaglou Theorem (see e.g.\ Theorem 3.15 in \cite{rudin2}) and the duality between $C(K)$ and $\mathcal{B} (K)$, the class of Borel measures over~$K$:

\begin{claim}\label{cl:BA}
Let $K$ be a compact Haussdorf space. Then $\Delta(K)$ is compact in the weak* topology.

\end{claim}
\paragraph{Upper and lower semicontinuity.}
Recall that a real function $f$ is called upper semicontinuous (u.s.c) if for every $\alpha\in \mathbb{R}$ the set $\{x: f(x)\ge \alpha\}$ is closed. 
Note that $\mathop{\lim\sup}_{x\to x_0}  f(x)\le f(x_0)$ for any $x_0$ in the domain of $f$.
Similarly, $f$ is called lower semicontinuous (l.s.c) if $-f$ is u.s.c. We will use the following fact:

\begin{claim}\label{cl:usc}
Let $K$ be a compact Haussdorf space and assume $E \subseteq K$ is a closed set. Consider the ``$\Delta(K)\to[0,1]$'' mapping  $T_{E}(\mu) = \mu(E)$. Then $T_{E}$ is u.s.c with respect to the weak* topology on $\Delta(X)$.
\end{claim}
\begin{proof}
This fact can be seen as a corollary of Urysohn's Lemma (Lemma 2.12 in \cite{rudin1}). Indeed, Borel measures are \emph{regular} (see definition 2.15 in \cite{rudin1}. Thus, for every closed set $E$ we have
\[\mu(E)= \inf_{\{U: E\subseteq U,~ \textrm{U is open}\}} \mu(U).\]
Fix a closed set $E$. Urysohn's Lemma implies that for every open set $U\supseteq E$, there exists a continuous function~$f_U\in C(K)$ such that $\chi_{E}\le f_{U}\le \chi_{U}$, where $\chi_{A}$ is the indicator function over the set $A$ (i.e. $\chi_{A}(x)=1$ if and only if $x\in A$).

Thus, we can write $\mu(E) = \inf_{\{U: E\subseteq U,~ \textrm{U is open}\}} \mu(f_{U})$, where $\mu(f_{U})=\EE_{x\sim \mu}[f_{U}]$. 
Now, by continuity of $f_U$, it follows that the mapping
$\mu \mapsto \mu(f_U)$ is continuous with respect to the weak* topology on $\Delta(X)$. Finally, the claim follows since the infimum of continuous functions is u.s.c.
\end{proof}

\paragraph{Sion's Theorem.} 
We next state the following generalization of Von-Neumann's Theorem for u.s.c/l.s.c payoff functions. 

\begin{theorem}[Sion's Theorem]\label{thm:sion}
Let $W$ be a compact convex subset of a linear topological space and~$U$ a convex subset of a linear topological space. If $F$ is a real valued function on $W\times U$ with
\begin{itemize}
\item $F(w,\cdot)$ is l.s.c and convex on $U$ and
\item $F(\cdot,u)$ is u.s.c and concave on $W$
\end{itemize}
then,
\[ \max_{w\in W}\inf_{u\in U} F(w,u)= \inf_{u\in U}\max_{w\in W} F(w,u)\]
\end{theorem}

\paragraph{Tychonof's space.}
The last notion we introduce is the topology we will use on $\{0,1\}^\X$. 
Given an arbitrary set $\X$, the space $\F=\{0,1\}^\X$ is the space of all functions $f:X\to \{0,1\}$. The product topology on $\F$ is the weakest topology such that for every $x\in \X$ the mapping $\Pi_x:\F\to\{0,1\}$, defined by $\Pi_{x}(f) = f(x)$ is continuous.

A basis of open sets in the product topology is provided by the sets $U_{x_1,\ldots,x_m}(g)$ of the form:
\[U_{x_1,\ldots, x_m}(g) = \{f: g(x_i)=f(x_i)~i=1,\ldots,m\},\]
where $x_1,\ldots, x_m$ are arbitrary elements in $X$ and $g\in \F$. 

A remarkable fact about the product topology is that
the space $\F$ is compact for any domain~$\X$  (see for example \cite{kelley2017general}). 
We summarize the above discussion in the following claim

\begin{claim}\label{cl:tychonof}
Let $\X$ be an arbitrary set and consider $\F=\{0,1\}^\X$ equipped with the product topology. Then $\F$ is compact  and $\Pi_x\in C(\F)$ for every $x\in X$, where $\Pi_x$ is defined as~$\Pi_{x}(f)= f(x)$. 
\end{claim}

\subsection{Two Technical Lemmas}
The proof of \cref{lem:minmax_infinite} follows from the following two Lemmas. Throughout the proofs we will treat $\D$ as a topological subpace in $\{0,1\}^\X$ with the product topology. We will also naturally treat $\Delta(\D)$ as a topological space equipped with the weak$^*$ topology.

\begin{lemma}[Analog of \cref{lem:vonneuman}]\label{lem:uppersemi}
Assume $\D\subseteq\{0,1\}^\X$ is closed and let $f: \D \to [0,1]$.
Assume that $\hat f$ is u.s.c (with respect to the weak* topology on $\Delta(\D)$) then $f$ is amenable.
\end{lemma}

\begin{lemma}[Analog of \cref{cor:bendavid}]\label{lem:bendavid_infinite}
Let $\D\subseteq\{0,1\}^\X$ be closed and let $\ell^*$ denote its dual Littlestone dimension. 
Then, there exists a deterministic online learner that receives labelled examples from the domain $\Delta(\D)$ such that for every sequence $(p_t,y_t)_{t=1}^T$ we have that:
%\[\sum_{t=1}^T |f_t(p_t)-y_t| - \min_{x\in \X}\sum_{t=1}^T |\EE_{d\sim p_t}[d(x)]-y_t| = O(\sqrt{\ell^* T\log T}).\]
\[ \regret(L) \le \sqrt{\frac{1}{2} \ell T\log T}\]
Moreover,  at each iteration $t$ the predictor, $\hat{f}_t$, used by $L$ is of the form $\hat{f}_t\left[\bar{d}\right]= \EE_{d\sim \bar{d}}(f_t(d))$ for some $f_t:\D\to [0,1]$. Finally, for every $t\le T$, if the sequence of observed examples $(\bar{d}_1,y_1),\ldots, (\bar{d}_{t-1},y_{t-1})$ all have finite support then $\hat{f}_t$ is u.s.c.
%
%Further, the functions $f_t$ chosen by the learner are always linear and u.s.c
\end{lemma}

We first show how to conclude the proof of \cref{lem:minmax_infinite} using these lemmas and later prove the two lemmas.
\paragraph{Concluding the proof of \cref{lem:minmax_infinite}.}
The proof follows directly from the two preceding Lemmas. Given a discriminating class $\D\subseteq \{0,1\}^\X$ there is no loss of generality in assuming $\D$ is closed,
since closing the class with respect to the product topology does not increase its dual LIttlestone dimension.

Now, take the learner $\A$ whose existence follows from \cref{lem:bendavid_infinite}. 
Since each $\hat f_t$ is u.s.c we obtain via \cref{lem:uppersemi} that each $f_t$ is also amenable.
%\qed

\paragraph{Proof of \cref{lem:uppersemi}.}
\cref{lem:uppersemi} extends \cref{lem:vonneuman} to the infinite case. 
Similar to the proof of \cref{lem:vonneuman} which hinges on Von-Neumann's Minmax Theorem, the proof here hinges on
Sion's Theorem which is valid in this setting.

Before proceeding with the proof we add the following notation: let $\mathbb{R}_{fin}^{\X}$ denote the space of real-valued functions $v:\X\to \mathbb{R}$ with finite support, i.e. $v(x)=0$ except for maybe a finite many $x\in \X$. We equip $\mathbb{R}_{fin}^{\X}$ with the topology induced by the $\ell_1$ norm, namely a basis of open sets is given by the open balls $U_{v,\epsilon}=\{u: \sum_{x\in \X} |v(x)-u(x)|<\epsilon\}$. $\mathbb{R}_{fin}(\X)$ is indeed a linear topological space (i.e.\ the vector addition and scalar multiplication mappings are continuous). Finally, define
\[ \Delta_{fin}(\X) := \{p\in \mathbb{R}_{fin}^{\X} : p(x)\ge 0 ~ \sum_{x\in \X}p(x)=1\}.\]

Next, let $f:\D\to[0,1]$ be such that $\hat{f}$ is u.s.c. Our goal is  to show that $f$ is amenable.
Set $F$ to be the following real-valued function over $\Delta(\D)\times \Delta_{fin}(\X)$:
\[F(\bar{d},p)=   \EE_{\bar{d}\sim d}\left[f(d)-\sum_{x\in \X}  p(x)x(d)\right]\]
%To result readily follows once we show we can apply Sion's Theorem over $F$. In particular we aim to show the following:
It suffices to show that
\begin{align}\label{eq:sion}
\max_{\bar{d}\in \Delta(\D)}\inf_{p\in \Delta_{fin}(\X)} F(\bar{d},x) = \inf_{p\in \Delta_{fin}(\X)} \max_{\bar d\in \Delta(\D)}F(\bar{d},p) 
\end{align}
Indeed, the assumption that \cref{it:if} does not hold implies in particular that 
\[\inf_{p\in \Delta_{fin}(\X)} \max_{d\in \Delta(\D)} F(\bar{d},p)
\ge \frac{\epsilon}{2}.\] \cref{eq:sion} then states that
\[
\max_{\bar{d}\in \Delta(\D)}\inf_{x\in \X} \EE_{d\sim \bar{d}}\left[f(d)-x(d)\right]\ge \frac{\epsilon}{2}.\]
which proves that \cref{it:else} holds.

\cref{eq:sion} follows by an application of \cref{thm:sion} on the function $F$.
Thus, we next show the premise of \cref{thm:sion} is satisfied by $F$.
Indeed, $W=\Delta(\D)$ is compact and convex, and $U=\Delta_{fin}(\X)$ is convex.
We show that $F(\cdot,p)$ is concave and u.s.c for every fixed $p\in \Delta_{fin}(\X)$: indeed, $F(\cdot,p)$ is in fact linear and therefore concave. We show that $F(\cdot,p)$ is u.s.c by showing that it is the sum of (i) a u.s.c function (i.e.\ $\EE_{d\sim \bar d}[f(d)]$) and (ii) finitely many continuous functions (i.e.\ $\sum_{x\in \X}  p(x)\EE_{d\sim \bar d}[x(d)]$).
Indeed, (i) by assumption $\hat f(\bar{d})=\EE_{d\sim \bar d}[f(d)]$ is u.s.c,
and (ii) by \cref{cl:tychonof}, the mapping $\Pi_{x}(d)$ is continuous for every $x\in \X$ which, by the definition of the weak* topology, 
implies that $\bar{d}\to \EE_{d\sim \bar{d}} \Pi_x(d) = \EE_{d\sim \bar{d}}\left[x(d)\right]$ is continuous.

Finally, because $\EE_{d\sim \bar{d}}[x(d)]\le 1$ is bounded, it follows that $F(\bar{d},\cdot)$ is linear and continuous in $p$ for every fixed $\bar{d}$: indeed treating $\hat{f}(\bar{d})$ and $\{\EE_{\bar{d} \sim d} \left[x(d)\right]\}_{x\in \X}$ as bounded constants, we have that: 
\[ F(\bar{d},p) = \hat{f}(\bar{d}) - \sum_{x\in X} p(x) \EE_{\bar{d}\sim d} \left[ x(d)\right]\]

\ignore{
Let $f$ be a linear u.s.c function that does not meet \cref{it:if} (otherwise, clearly $f$ is amenable). 
Define the set
\[E_{x}=\{ \mu \in \Delta(\D): x(\mu)-f(\mu)\le \frac{\epsilon}{2}\}.\]
 It suffices to prove that there exists a measure $\mu^*$ such that $\mu^*\in \cap_{x\in \X} E_{x}$. Indeed, for such a measure we have that for every $x\in \X$, $x(\mu^*)-f(\mu^*) \le \frac{\epsilon}{2}$.
Thus, proving \cref{lem:uppersemi} reduces to showing that $\cap_{x\in X} E_{X}$ is non-empty. 

The proof will follow from the next two statements:
\begin{enumerate}
    \item For every $x\in X$ the set $E_{x}$ is closed in the weak*-topology.
    \item For every finite set of elements $x_1,\ldots, x_m \in \X$ we have that $\cap_{x_i} E_{x_i}\ne \emptyset$.
\end{enumerate}
Indeed, once these two statements are established then, because $\D$ is closed by assumption, $\D$ is compact by \cref{cl:tychonof}. Then \cref{cl:BA} states that 
$\Delta(\D)$ is compact and it follows that the set $\cap_{x\in \X}E_{x}$ is non-empty as required.

\paragraph{Proving that $E_{x}$ are closed.}
Using \cref{cl:tychonof} we have that for every $x\in \X$, $G_x\in C(\D)$ and hence for every $x\in \X$ the mapping $T_{G_x}(\mu)=G_x(\mu)=\EE_{d\sim \mu}[d(x)]$ is a continuous mapping over $\Delta(\D)$, by definition of the $w^*$ topology.

Because $f$ is u.s.c, we have that $x-f$ is an l.s.c function, hence $E_{x}$ is a closed set.

\paragraph{Proving that finite intersections are non-empty.} Fix $x_1,\ldots, x_m$ let denote $\X_m=\{x_1,\ldots, x_m\}$ and define $\D_m$ a finite set of discriminators over $\X_m$ that consists of all restrictions of $\D$ to $\X_{m}$.

Set $\tilde{d}_1,\ldots, \tilde{d}_k$ to be the elements of $\D_m$. Note that each element in $\D_m$ can be considered as an equivalence class over the elements in $\D$ (via the identification $d\equiv d'$ iff $d_{|\X_m}=d'_{|\X_m}$). Denote by $\pi:\D\to \D_m$ the mapping that maps $d\in \D$ to its equivalence class. 

We next identify a mapping $f^{(m)}$ over $\D_m$ via the rule:
\begin{align}\label{eq:max} f^{(m)}(\tilde{d})= \max_{\{d: \pi(d)=\tilde{d}\}}f(d),\end{align}
We again extend $f^{(m)}$ to operate over all distributions in $\Delta(\D_m)$ via linearity: i.e. $f^{(m)}(p)= \EE_{\tilde{d}\sim p}[ f^{(m)}(\tilde{d})]$.

One can observe that if $f$ doesn't meet \cref{it:if} w.r.t $\X$ and $\D$ then $f^{(m)}$ doesn't meet \cref{it:if} w.r.t $\X_m$ and $\D_m$. But since this is the finite case, we obtain via \cref{lem:vonneuman} that $f^{(m)}$ must meet \cref{it:else} w.r.t to $\D_m$ and $\X_m$. In particular we have that for some $p\in \Delta(\D_m)$, for every $x\in \X_m$ 
\[x(p) - f^{m}(p) \le \frac{\epsilon}{2} .\]
Now we define $\mu\in \Delta(\D)$ as follows: for every $\tilde{d}_i\in \D_m$ pick $d_i\in \D$ that is a maximizer of \cref{eq:max} and set $\mu(d_i)=p(\tilde{d}_i)$. One can observe that
\[x_i(\mu) - f(\mu) \le \frac{\epsilon}{2},\]
for every $x_i\in \X_m$. This concludes the proof that $\cap_{x_1,\ldots,x_m} E_{x_i}$ is non empty.
%\qed

}
\paragraph{Proof of \cref{lem:bendavid_infinite}.}
\cref{lem:bendavid_infinite} follows from a close examination of the proof provided in \cite{bendavid} for \cref{thm:bendavid} and the extension to \cref{cor:bendavid}.

The fact that the learner outputs a predictor of the form $\hat{f}_t=\EE_{\bar{d}\sim d}\left[f_t(d)\right]$ follows by construction in \cref{cor:bendavid}. So, it suffices to show that the $f_t$'s can be chosen to be u.s.c.
Call a function $s:\D\to \{0,1\}$ an SOA-type function if there exists a hypothesis class $\H\subseteq \X$ such that

\[s(d)=
\begin{cases} 
0 & \ldim(\H|_{(d,0)})=\ldim(H)\\
1 & \mathrm{else}
\end{cases}
\]
where $H|_{(d,0)}=\{h\in H\}:~ h(d)=0\}$.

In the proof by \cite{bendavid} of \cref{thm:bendavid} the authors construct an online learner which at each iteration uses a 
randomized predictor (i.e.\ a distribution over predictors). One can observe and see that this randomized predictor only uses SOA-type function: namely, the algorithm holds, at each iteration, a distribution $q_t$ over a finite set of SOA type functions $\{s_k\}$, and at each iteration picks the prediction made by $s_k$ with probability $q_t(s_k)$.

The extension in \cref{cor:bendavid} of this predictor to the domain $\Delta(\D)$ is done by choosing:
\[ f_t(d) =  \EE_{\bar{d}_{1:T}} \left[  \EE_{s\sim L_0} \left[s(d) | d_1,\ldots, d_{t-1}\right]\right] =\EE_{\bar{d}_{1:T}} \left[ \sum q_t(s_k) s_k(d) | d_1,\ldots, d_{t-1}\right]\]

Namely, the choice of $f_t$ is the expectation over the algorithm's prediction, taking expectation both over the choice of the algorithm and over the sequence of observations. $d_1,\ldots, d_{t-1}$, drawn according to $\bar{d}_1,\ldots, \bar{d}_{t-1}$. Now because $\bar{d}_1,\ldots \bar{d}_{t-1}$ all have finite support we can summarize these expectations and write:
\[ f_t= \sum \lambda_k s_k,\] for some choice of SOA-type functions and weights $\lambda_k\ge 0$.

Since the sum of u.s.c functions is u.s.c and since the multiplication of a u.s.c function with positive scalar is u.s.c, it is enough to prove that every SOA-type function $s$ induces an u.s.c function over $\Delta(\D)$ via the identification $\mu \mapsto  \mu\left(\{d: s(d)=1\}\right)$. 
By \cref{cl:usc} it is enough to show that the set $s^{-1}(0)$ is open.
To this end we show that for every $d\in s^{-1}(0)$ there is an open neighborhood of $d$ which is contained in $s^{-1}(0)$.
Indeed, if $d\in s^{-1}(0)$, then there exist $x_1,\ldots, x_{2^{\ell}}$ that $d(x_i)=0$ for all~$i$, and they shatter a tree. 
Consider the open neighborhood of $d$ defined by~$U=\cap_{i}\{d: d(x_i)=0\}$. 
$U\subseteq s^{-1}(0)$ since if there were $d'\in U$ such that $s(d')=1$
then $\ldim(\H|_{(d',0)}) < \ldim(\H)=\ell$. However, since $d'\in U$ then $x_1,\ldots,x_{2^\ell}\in \H|_{(d',0)}$
and they shatter a tree of depth $\ell$ which is a contradiction.
%\qed

\ignore{
\subsection{Extension -- old}
We next extend the proof of correctness of \cref{alg:main} to the general case of infinite \good classes. The proof generally follows the same lines. In fact we only used finiteness in \cref{lem:minmax}, which ensure that if the condition in \cref{it:if} is not met then we can find $\bar{d}$ as required in \cref{it:else}. \cref{lem:minmax_infinite}, then, extends \cref{lem:minmax} to the general infinite case.

The proof relies on certain facts from Functional Analysis which we will briefly outline and we refer the reader to \cite{rudin1} for further reading.

\begin{lemma}\label{lem:minmax_infinite}
Let $\D$ be a \good discriminating class and let $F\in [0,1]^{\D}$ be such that for every $\mu\in \Delta(\X)$ there exists $d\in \D$ such that
\[ F(d)> \mu(d)+\epsilon.\]
Then there exists $\bar{d}\in \Delta(\D)$ such that for every $\mu$:
\[F(\bar{d}) \ge \mu(\bar{d})+\epsilon\]
\end{lemma}

\paragraph{preliminaries.}
Before diving into the proof present the general notions and facts that we will need from real functional analysis:

Recall that we assume that $\X$ is a compact Haussdorf space. The weak* topology over $\Delta(\X)$ is defined as the weakest topology so that for any continuous function $f$ over $\X$ we have that the following mapping is continuous
\[T_{f}(\mu) = \int f(x) d\mu(x) .\]
For separable $\X$, this means that a sequence $\mu_n$ of measures converges to $\mu$ if and only if for every continuous function $f$ we have that $\mu_n(f)\to \mu(f)$. 

The following two facts about the weak* topology will be necessary for our proofs

\begin{enumerate}
\item The set $\Delta(\X)$ is compact in the weak* topology: This is a direct consequence of Banach--Alaglou theorem (Theorem 3.15 in \cite{rudin2}) and the duality between $C(\X)$-- the space of continuous functions over $\X$, and the space $\mathcal{B}(\X)$ of Borel measures (Theorem 2.14 in \cite{rudin1}).

\item For every closed set $E\subseteq \X$ the mapping $T_{E}(\mu) =\mu(E)= \int_{x\in E} d\mu(x)$ is \emph{upper semi--continuous}. i.e. for every $\alpha \in \mathbb{R}$ we have that the following set is closed in the weak* topology:

\[U_{E,\alpha}=\{\mu: T_{E}(\mu)\ge \alpha\}.\]
This fact can be seen as a corollary of Urysohn's Lemma (Lemma 2.12 in \cite{rudin1}), regularity of Borel measures (Definition 2.15 in \cite{rudin1}), and the fact that the infinmum of continuous functions is an upper semicontiuous function.
\end{enumerate}
\paragraph{Proof of \cref{lem:minmax_infinite}.}
Let $\mu\in \Delta(\X)$, then by assumption there is $d\in \D$ s.t.:

\[ \mu(d) < F(d) -\epsilon.\]
Thus we obtain that

\begin{align}\label{eq:finite_intersection}\cap_{d\in \D}\{\mu: \mu(d) \ge F(d)-\epsilon\}=\emptyset.\end{align}
By compactness we have that there is a finite set $d_1,\ldots, d_n$ such that
\[\cap_{d_1,\ldots, d_n}\{\mu: \mu(d_i) \ge F(d_i)-\epsilon\}=\emptyset.\]
Now consider the vector $v_{F}\in \mathbb{R}^{n}$ such that $v_{F}(i)= F(h_i)$ and the set $K\subseteq \mathbb{R}^n$:
\[K=\{ x\in \mathbb{R}^n: \exists \mu \in \Delta(x) \mu(h_i)= x(i)\}.\]
Then we have the following

\begin{align*} 
\max_{\|\beta\|_1\le 1,\beta \succ 0} \min_{x\in K}\sum \beta(i)v_F(i)- \sum \beta(i) x(i) &= \min_{x\in K}\max_{\|\beta\|_1\le 1,\beta\succ 0} \sum \beta(i)v_F(i)- \sum \beta(i) x(i) 
 \\&= \min_{\mu\in \Delta(\X)} \left(\max_{i} F(d_i)-\mu(d_i)\right)
  \\& \ge \epsilon
 \end{align*}}	%ADD FOR FULL VERSION

\end{document}